\documentclass[sigconf]{acmart}
\AtBeginDocument{%
  }

\setcopyright{none}

\copyrightyear{}
\acmYear{}
\acmDOI{}
\acmISBN{}

\acmConference[KDD '26]{ACM SIGKDD Conference on Knowledge Discovery and Data Mining}
{August 2026}{Jeju Island, Republic of Korea}
\acmBooktitle{Accepted to KDD 2026}

\usepackage{amsmath,amsfonts,amsthm,bm}
\usepackage{xcolor}
\usepackage{fontawesome}

\def\eqref#1{(\ref{#1})}

\def\1{\bm{1}}

\def\rx{{\textnormal{x}}}

\def\rvx{{\mathbf{x}}}

\def\vone{{\bm{1}}}

\def\vtheta{{\bm{\theta}}}

\def\vphi{{\bm{\phi}}}

\def\vc{{\bm{c}}}

\def\vq{{\bm{q}}}

\def\vx{{\bm{x}}}
\def\vy{{\bm{y}}}
\def\vz{{\bm{z}}}

\def\mA{{\bm{A}}}
\def\mB{{\bm{B}}}

\def\mD{{\bm{D}}}

\def\mH{{\bm{H}}}
\def\mI{{\bm{I}}}

\def\mK{{\bm{K}}}
\def\mL{{\bm{L}}}

\def\mP{{\bm{P}}}
\def\mQ{{\bm{Q}}}

\def\mV{{\bm{V}}}
\def\mW{{\bm{W}}}
\def\mX{{\bm{X}}}
\def\mY{{\bm{Y}}}
\def\mZ{{\bm{Z}}}

\DeclareMathAlphabet{\mathsfit}{\encodingdefault}{\sfdefault}{m}{sl}
\SetMathAlphabet{\mathsfit}{bold}{\encodingdefault}{\sfdefault}{bx}{n}

\def\gD{{\mathcal{D}}}

\def\gG{{\mathcal{G}}}
\def\gH{{\mathcal{H}}}

\def\gL{{\mathcal{L}}}

\def\gN{{\mathcal{N}}}

\def\gS{{\mathcal{S}}}

\def\gV{{\mathcal{V}}}

\def\gX{{\mathcal{X}}}
\def\gY{{\mathcal{Y}}}

\def\sP{{\mathbb{P}}}

\newcommand{\E}{\mathbb{E}}

\newcommand{\R}{\mathbb{R}}

\newcommand{\abs}[1]{\left\lvert #1 \right\rvert}

\newcommand{\norm}[1]{\left\lVert#1\right\rVert}
\newcommand{\T}{\top}
\newcommand{\F}{\textup{F}}
\newcommand{\tr}{\textup{tr}}

\usepackage{listings}
\usepackage{color}

\usepackage{enumitem}
\usepackage{algorithm} 
\usepackage{algpseudocode} 
\usepackage{graphicx}
\usepackage{subfigure}
\usepackage{multirow}
\usepackage{booktabs}

\def\yv{{y_s(v)}}

\def\xvs{{\rvx_s(v)}}
\def\xvsp{{\{\rvx_s(u)\}_{u \sim v}}}
\def\xvt{{\rvx_t(v')}}

\begin{document}
\title{Enhancing Node-Level Graph Domain Adaptation by Alleviating Local Dependency}

\author{Xinwei Tai}
\authornote{Both authors contributed equally to this research.}
\email{tttxw2004@gmail.com}
\orcid{0009-0007-4933-2224}
\affiliation{%
  \institution{Huazhong University of Science and Technology\\ School of Cyber Science and Engineering\\ Hubei Key Laboratory of Distributed System Security\\ Hubei Engineering Research Center on Big Data Security}
  \city{Wuhan}
  \country{China}
}
\affiliation{%
  \institution{Zhongguancun Academy}
  \city{Beijing}
  \country{China}
}

\author{Dongmian Zou}
\authornotemark[1]
\email{dongmian.zou@duke.edu}
\orcid{0000-0002-5618-5791}
\affiliation{%
  \institution{Zu Chongzhi Center, Digital Innovation Research Center\\ Duke Kunshan University}
  \city{Kunshan}
  \country{China}
}

\author{Hongfei Wang}
\authornote{Corresponding author.\\ This work was supported by the National Natural Science Foundation of China (NSFC) under Grant No. 12301117, Grant No. 62372198 and Grant No. 62172173.}
\email{hongfei@hust.edu.cn}
\orcid{0000-0001-5377-5924}
\affiliation{%
  \institution{Huazhong University of Science and Technology\\ School of Cyber Science and Engineering\\ Hubei Key Laboratory of Distributed System Security\\ Hubei Engineering Research Center on Big Data Security}
  \city{Wuhan}
  \country{China}
}


\begin{abstract}
Recent years have witnessed significant advancements in machine learning methods on graphs. However, transferring knowledge effectively from one graph to another remains a critical challenge. This highlights the need for algorithms capable of applying information extracted from a source graph to an unlabeled target graph, a task known as unsupervised graph domain adaptation (GDA). One key difficulty in unsupervised GDA is conditional shift, which hinders transferability. In this paper, we show that conditional shift can be observed only if there exists local dependencies among node features. To support this claim, we perform a rigorous analysis and also further provide generalization bounds of GDA when dependent node features are modeled using markov chains. Guided by the theoretical findings, we propose to improve GDA by decorrelating node features, which can be specifically implemented through decorrelated GCN layers and graph transformer layers. Our experimental results demonstrate the effectiveness of this approach, showing not only substantial performance enhancements over baseline GDA methods but also clear visualizations of small intra-class distances in the learned representations. Our code is available at  \url{https://github.com/TechnologyAiGroup/DFT}.
\end{abstract}



\received{20 February 2007}
\received[revised]{12 March 2009}
\received[accepted]{5 June 2009}

\maketitle

\section{Introduction}\label{sec:intro}
Graph neural networks (GNNs) have recently emerged as a fundamental framework for extracting information from graph-structured data. In many real-world applications, it is often crucial to transfer knowledge from one graph to another. However, graphs often exhibit distinct structures and topologies, inevitably leading to variations in their underlying data distributions. Indeed, it has been shown difficult to obtain generalizable models for learning in applications such as social networks~\cite{fan2019graph}, citation networks~\cite{kipf2016semi}, or biological interaction networks~\cite{han2019gcn}, where each graph represents inherently complex real-world systems and encapsulates unique structural characteristics and feature distributions. This underscores the importance of node-level graph domain adaptation (GDA), where the same predictive tasks are performed on both a source graph domain and a target graph domain, despite differences in their graph structures and associated feature distributions. In this paper, we focus on the unsupervised setting, where no training label is available in the target domain.

In GDA, the challenge of transferring knowledge is amplified by the fact that node features are not independent, which is in contrast to traditional domain adaptation problems. Indeed, traditional domain adaptation approaches typically rely on generalization bounds that assume independent and identically distributed (i.i.d.) data within each domain~\cite{ganin2015unsupervised}. Pioneering works on GDA, such as UDAGCN~\cite{wu2020unsupervised} and AdaGCN~\cite{dai2022graph}, have focused on effective feature extraction on graphs, but have not explicitly addressed the issue of feature dependency. Indeed, their reliance on the GCN model~\cite{kipf2016semi} inherently exacerbates this issue as GCN further introduces local dependencies in node representations. 

More recent studies~\cite{you2023graph, zhu2023explaining} have noted that the standard ``covariate shift assumption, where the label distribution conditioned on input features remains unchanged across domains, is often insufficient to explain distribution shifts in graph domain adaptation. In practice, conditional shift, where the label distribution conditioned on input graphs differs between the source and target domains, may also be present. For example, in high energy physics, conditional shifts can arise from differences in pileup levels~\cite{liu2023structural,liu2024pairwise}. While covariate shift is only one of several assumptions used in domain adaptation and there exist works that discuss other settings~\cite{garg2023rlsbench, tsai2024proxy}, it is widely satisfied by many real datasets and used in many GDA papers~\cite{dai2022graph, wu2020unsupervised, liu2023structural, you2023graph, wu2023noniid, liu2024rethinking, liu2024pairwise}. For instance, in citation networks, where node features are often sparse bag-of-words vectors~\cite{dai2022graph}, the label distribution given features is expected to remain largely consistent across datasets. The problem is that assuming covariate shift does not preclude the presence of conditional shift. The latter can still arise due to the graph structure itself, where node features and representations are inherently dependent on their local neighborhoods.

These considerations naturally lead to the following research questions on conditional shifts for GDA:
\textbf{RQ1.} What does the presence of conditional shifts reveal about dependency of node representations for the underlying graph data?
\textbf{RQ2.} How does dependency of node representations impact the performance of GDA? \textbf{RQ3.} How can we improve performance of GDA methods accordingly? 

In this paper, we address the above questions through a rigorous study of the relationship between node representation dependencies and GDA performance. Our main contributions are as follows:
\begin{itemize}[listparindent=0pt, leftmargin=*]
    \item We provide a theoretical analysis of unsupervised node-level GDA. We show that the presence of conditional shifts necessarily implies interdependent node representations.
    \item We derive generalization bounds to understand how local dependencies hinder generalization performance. We further analyze why GCN is not an ideal backbone model by demonstrating that GCN propagation amplifies feature interdependency.
    \item We show that simply improving GCN layers by introducing de-correlation leads to noticeable improvements in the classical UDAGCN model, resulting in state-of-the art performance.
\end{itemize}

\section{Related Works}\label{sec:related_works}
\subsection{GNN and Generalizability}
Recent research on GNNs has widely adopted the neural message-passing scheme pioneered by \citet{gilmer2017neural}. In this scheme, each GNN layer aggregates and processes information from the neighbors of each node. Notable message-passing GNNs include well-known architectures such as ~\cite{kipf2016semi, velivckovic2017graph, xu2018powerful}. On the other hand, transformers~\cite{vaswani2017attention} have gained widespread adoption and have been extended to accommodate graph-structured data~\cite{rong2020self, hussain2022global, ying2021transformers, kreuzer2021rethinking, min2022transformer,dwivedi2020generalization}. Unlike previous attention scheme in GAT~\cite{velivckovic2017graph} which also yields message passing, graph transformers employ self-attention. Their successes lie in their ability to capture global patterns, effectively addressing oversmoothing~\cite{cai2020note, chen2020measuring} and oversquashing~\cite{alon2021bottleneck} issues observed in message-passing GNNs. 

To understand the generalizability of GNNs, recent works such as~\cite{esser2021learning} and \cite{tang2023towards} derived upper bounds for their VC dimensions and Rademacher complexities. However, these studies have primarily considered generalization as an intrinsic property of GNNs and thus do not directly extend to GDA schemes. Recent research has also delved into the realm of out-of-distribution generalization in GNNs~\cite{li2022outofdistribution}. While various techniques have emerged to enhance out-of-distribution generalization~\cite{zhu2021shift, wu2022handling}, they do not inherently constitute comprehensive GDA strategies. On the other hand, the oversmoothing phenomenon manifests as features become increasingly similar as GNNs progress to deeper layers, which is known to hinder GNN generalization~\cite{oono2020Graph}. A related issue is overcorrelation~\cite{jin2022feature, liu2023enhancing}. However, the term ``correlation'' there specifically applies to different feature dimensions, not different nodes.

\subsection{Graph Domain Adaptation}
Traditional domain adaptation assumes i.i.d. data points~\cite{ben2006analysis, ben2010theory, tzeng2014deep,long2015learning, ganin2015unsupervised, ganin2016domain, zhang2019bridging, chang2022unified}. However, node features on graphs show more complex interdependency relations. To address GDA problems, recent works adopted various techniques to facilitate aligning the latent distributions in both source and target domains. For instance, \citet{song2020domain, dai2022graph} used different adversarial approaches with GNN feature extractors. \citet{wu2020unsupervised} used attention mechanism on both graph adjacency matrices and point-wise mutual information matrices to effectively extract latent information. \citet{cai2021graph} assumed that graph datasets are generated using independent latent variables and designed various GNN modules to learn these latent variables. In a more recent development, \citet{wu2023noniid} used Weisfeiler-Lehman graph isomorphism to derive a generalization bound and designed a method by minimizing kernel distance induced by GNN. \citet{liu2024rethinking} leveraged the intrinsic generalizability of GNNs in GDA tasks. GDA methods have also been used in other settings such as semi-supervised learning~\cite{zhu2021shift} and reinforcement learning~\cite{wu2022handling, yang2022learning}. However, these methods have primarily focused on the effective utilization of GNNs and may overlook the inherent interdependency nature of graph-structured data.

Another line of recent research ascribed the difficulties encountered in GDA to shifts of conditional distributions and proposed various strategies to alleviate these issues. Specifically, \citet{liu2023structural} scrutinized the shift in graph structures conditioned on class labels. More recently, \citet{liu2024pairwise} used edge weights and label weights to align the shifts in conditional feature distribution and label distribution. \citet{you2023graph} delved into the shift of label distributions conditioned on graph structures, and \citet{zhu2023explaining} explored shifts in label distribution conditioned on node features. Unlike these works, we analyze that local dependency of node representations is indicated by observing conditional shifts and therefore propose approaches specifically targeted at reducing such dependency.

\section{Main Idea and Findings}\label{sec:analysis}
\subsection{Settings}
We assume that node features and their labels are drawn from the same representation space $\gX \subset \R^D$ and the same label space $\gY$ with a distance metric $d_\gY$, respectively, for both source and target domains. We assume a groundtruth labeling function $h_0: \gX \to \gY$. We denote the data distribution in the source and target domains to be $\mu_s$ and $\mu_t$, respectively. More specifically, these represent the marginal distributions, which describe how the input data is distributed in each domain ignoring any associated labels or outputs. Assume that the training data in the source domain have an underlying attributed graph $\gG$ with vertex set $\gV$, where $\abs{\gV} = N_s$. Both node features and labels $\{\vx_s(v), y_s(v)\}_{v \in \gV}$ are available. On the other hand, in the target domain, we have another attributed graph $\gG'$ with vertex set $\gV'$, where $\abs{\gV'} = N_t$. We consider the unsupervised setting, which means that in the target domain, only node features $\{\vx_t(v')\}_{v' \in \gV'}$ are available but not their labels. We assume that the vertices are indexed and we can represent the features in feature matrices $\mX_s$ and $\mX_t$. We remark that, in general, the node features are non-i.i.d. since they are dependent according to the underlying graphs. We assume that the learned predicting function $h: \gX \to \gY$ is taken from a hypothesis class $\gH$. In each of domain, the true error functions on $\gH$ is
\begin{equation}\label{eq:def_erh}
    \epsilon_r(h) := \E_{\mu_r} \left[ d_\gY(h_0, h) \right], ~h \in \gH, \quad r = s,  t.
\end{equation}
In the source domain, the empirical error function is
\begin{equation}\label{eq:def_eshhat}
    \hat{\epsilon}_s(h) := \frac{1}{N_s} \sum_{v \in \gV} d_\gY \left( h(\vx_s(v)) , h_0(\vx_s(v)) \right).
\end{equation}
More notations and a tabular summary of all notations are provided in  Table~\ref{tab:Table 1}. 
\begin{table}[t]
    \centering
    \caption{List of notations}
    \label{tab:Table 1}
    \scalebox{0.9}{
    \begin{tabular}{c|c}
        \toprule
        Notation & Description \\
        \midrule
        $\gG$ & source graph\\
        $\gG'$ & target graph\\
        $\gV$, $\gV'$ & node set of $\gG$, $\gG'$\\
        $N_s$, $N_t$ & number of nodes in $\gG$, $\gG'$\\
        $\mX_s$, $\mX_t$ & feature matrices of $\gG$, $\gG'$\\
        $\mA_s$, $\mA_t$ & adjacency matrices of $\gG$, $\gG'$\\
        $\mP_s$, $\mP_t$ & PPMI matrices of $\gG$, $\gG'$\\
        \midrule
        $\gX$, $\gY$ & sample space and label space\\
        $v \sim u$ & $v$ is connected to $u$ in the graph\\
        $\vx_s(v)$, $\vx_t(v')$ & node features\\
        $\mu_s$, $\mu_t$ & marginal distributions for node features\\
        $P_s, P_t, p_s, p_t$ & probabilities and densities\\
        $D$ & feature dimension\\
        $d$ & VC dimension\\
        $C$ & number of classes\\
        $y_s(v)$, $y_t(v')$ & node class labels\\
        $\gH$ & hypothesis class\\
        $\epsilon_s$, $\epsilon_t$ & true error functions\\
        $\hat{\epsilon}_s$ & empirical error function\\
        $\epsilon^*$ & minimum sum of true errors in both domains\\
        $t_\textup{mix}$ & mixing time\\
        $\Lambda(\gG)$ & forest complexity of $\gG$\\
        \midrule
        $\mZ_s$, $\mZ_t$ & source and target graph node representations\\
        $\gamma$, $\lambda_1$, $\lambda_2$ & hyperparameters in decorrelated GCN\\
        $\lambda_t$, $\lambda_\textup{critic}$, $\lambda_\textup{gp}$ & hyperparameters for various losses\\
        \midrule
        $\norm{\cdot}_{\operatorname{F}}$ & Frobenius norm\\
        $\circ$ & Hadamard product\\
        \bottomrule
    \end{tabular}
    }
\end{table}

\subsection{Observation of Conditional Shift Implies Interdependency}
In domain adaptation, one widely used assumption is the covariate shift assumption, where $\sP_s(y \vert \rvx) = \sP_t(y \vert \rvx)$ for $\rvx \in \gX$ and $y \in \gY$. In GDA, covariate shift is still valid for many real data, as evidenced in~\cite{dai2022graph, wu2020unsupervised, liu2023structural, you2023graph, wu2023noniid, liu2024rethinking, liu2024pairwise}. To validate the adoption of this assumption, we further present detailed empirical analysis in Appendix~\ref{app:covariate_shift}.

For graph data, another natural assumption is conditional shift, which states that $\sP_s(y \vert \gG) \neq \sP_t(y \vert \gG)$~\cite{you2023graph}. Since $\gG$ contains information about not only node representations but also edge connections, the conditioning factors differ in these two assumptions. We analyze the implications of conditional shift as follows.

In a graph, a node $v$ is directly influenced by its neighbors via edges. This structural dependency leads to a fundamental difference between standard domain adaptation and GDA. Specifically, we can represent the effect of $\gG$ according to the edges: 
\begin{equation}\label{eq:cond_shift_observ}
\begin{aligned}
    & \sP_s ( y_s(v) = y ~\big\vert~ \rvx_s(v) = \vx, \{\rvx_s(u)\}_{u \sim v} \big) \neq \\ & \qquad \sP_t \big( y_t(v') = y ~\big\vert~ \rvx_t(v') = \vx, \{\rvx_t(u')\}_{u' \sim v'} \big),
\end{aligned}
\end{equation}
where $v \in \gV$ and $v' \in \gV'$ are nodes from source and target graphs, and ``$\sim$'' denotes edge connection. 
An immediate consequence of \eqref{eq:cond_shift_observ} is the interdependency of node features, stated as follows.

\begin{theorem}\label{lemma:conditional_shift}
Under the covariate shift assumption, if \eqref{eq:cond_shift_observ} holds, then the node features cannot be independently sampled.
\end{theorem}
\begin{proof}
It suffices to prove the contrapositive. Suppose the node features $\{\rvx_s(v)\}_{v \in \gV}$ and $\{\rvx_t(v)\}_{v' \in \gV'}$ are both independently sampled. Then, by Bayes' Theorem, 
\begin{align*}
    &\quad P_s \big( \yv ~\big\vert~ \xvs, \xvsp \big) \nonumber \\
    &= \frac{p_s \big( \xvsp ~\big\vert~ \xvs, \yv \big) p_s \left( \xvs, \yv \right)}{p_s \big( \xvs, \xvsp \big)} \\
    &= \frac{p_s \big( \xvsp ~\big\vert~ \xvs, \yv \big) }{p_s \big( \xvs, \xvsp \big)} \cdot p_s \big( \xvs \big) P_s \big( \yv ~\big\vert~ \xvs \big). \label{eq:bayes}
\end{align*}
When $\{\rvx_s(v)\}_{v \in \gV}$ is independent, 
\begin{align*}
    &\quad \frac{p_s \big( \xvsp ~\big\vert~ \xvs, \yv \big) p_s \big( \xvs \big)}{p_s \big( \xvs, \xvsp \big)} \nonumber \\
    &= \frac{p_s \big( \xvsp \big) p_s \big( \xvs \big)}{p_s \big( \xvs, \xvsp \big)} = 1,
\end{align*}
which reduces to 
\begin{equation*}
    P_s \big( \yv ~\big\vert~ \xvs, \xvsp \big) = P_s \big( \yv ~\big\vert~ \xvs \big).
\end{equation*}
This means that the left-hand side of \eqref{eq:cond_shift_observ} is equal to $P_s ( \yv = y ~\vert~ \xvs = \vx )$.
Similarly, in the target domain, when $\{\rvx_t(v)\}_{v' \in \gV'}$ is independent, the right-hand side of \eqref{eq:cond_shift_observ} is equal to $P_t ( y_t(v') = y ~\vert~ \xvt = \vx )$.
By the covariate shift assumption, $\{\rvx_t(v)\}_{v' \in \gV'} = P_t ( y_t(v') = y ~\vert~ \xvt = \vx )$, which implies that the left-hand side and the right-hand side of \eqref{eq:cond_shift_observ} are equal, contradicting the assumed conditional shift in \eqref{eq:cond_shift_observ}.
\end{proof}

Theorem~\ref{lemma:conditional_shift} establishes that assuming conditional shift inherently leads to feature interdependencies, which must be accounted for in GDA. In the following section, we analyze how these interdependencies influence generalization bounds.

\subsection{Interdependency Leads to Poorer Generalization Bounds}
We now analyze how interdependencies among graph nodes hinder generalization. We assume that node features are sampled via a markov chain. This assumption not only introduces interdependencies but also aligns with the common use of random walk samplers in graph-based learning methods. Notably, previous GDA methods such as \cite{wu2020unsupervised}, incorporate both graph adjacency matrices and a matrix derived from random walks to capture long-range relations. In this context, we assume that the invariant measure of the markov chain is equal to the marginal distribution, and derive generalization bounds as a function of the mixing time \cite[Chapter 4.5]{levin2017markov} of the markov chain. We state the result in the following theorem, of which the proof is presented in Appendix~\ref{app:proof_inequality}.

\begin{theorem}\label{thm:dep_graph_markov}
Suppose the hypothesis class $\gH$ has a VC-dimension of $d$ and comprises $K$-Lipschitz functions.
Let
\begin{equation}\label{eq:def_lambda}
    \epsilon^* = \min_{h \in \gH} \epsilon_s(h) + \epsilon_t(h),
\end{equation}
where $\epsilon_s(h)$ and $\epsilon_t(h)$ are defined as in \eqref{eq:def_erh}.
Let $N_s$ data points be sampled following a markov chain with invariant measure $\mu_s$ and a mixing time of $t_{\textup{mix}}$. Then with a probability of at least $1-\delta$,
\begin{equation}\label{eq:bd_mixing}
\begin{aligned}
    \epsilon_t(h) &\leq \hat{\epsilon}_s(h) + \sqrt{\frac{8d \log(e N_s / d)}{N_s}} + \sqrt{\frac{18 \log(2/\delta) t_{\textup{mix}}}{N_s^2}} + \\
    &\qquad 2K W_1 (\mu_s, \mu_t) + \epsilon^*.
\end{aligned}
\end{equation}
\end{theorem}

The mixing time $t_\textup{mix}$ serves as a measure of how rapidly the markov chain converges to its invariant measure. We remark that if the source data are independent, then $t_\textup{mix} = 0$. More generally, the closer the chain is to an independent process, the smaller the mixing time, leading to a tighter generalization bound with a reduced right-hand side in \eqref{eq:bd_mixing}. This bound suggests that minimizing interdependencies among node representations is crucial for improved generalization. 

Another way of modeling interdependencies of graph data is through dependency graphs \cite{janson2004large}. The corresponding analysis is presented in Appendix~\ref{app:generalization}, which yields a conclusion analogous to the above theorem.

We emphasize that the dependencies need only be alleviated at the representation level rather than at the raw input level. This is consistent with traditional domain adaptation methods, where distributions are aligned only at the representation level \cite{ganin2015unsupervised}. In the following section, we argue that the widely used message passing scheme of graph neural networks can produce interdependent representation in this sense.

\subsection{Message Passing Can Lead to Higher Interdependency}
For a simple analysis, consider the following message-passing scheme similar to GCN: $\mH^{(k+1)} = \sigma ( \tilde{\mA} \mH^{(k)} \mW )$, where $\mH^{(k)}$ is the matrix of node feature in the $k$-th layer, $\tilde{\mA}$ is the adjacency matrix augmented with self loops: $\tilde{\mA} = \mA + \mI$, $\mW$ is the learnable weight matrix and $\sigma$ is the activation function. The next theorem shows that, even if the raw node features are i.i.d., a very simple GNN could produce correlated representations. The proof is presented in Appendix~\ref{app:proof_energy}.
\begin{theorem}\label{thm:high_corr}
Suppose $\mW = \mI$ and $\sigma = \operatorname{Id}$ in the above propagation. Given input data $\mX\in \R^{N\times D}$ whose rows are i.i.d. sampled from $\gN(0, \mI)$, the feature correlation measured by $\E \left[ \norm{\mH^{(k)} \mH^{(k) \T}}_{\operatorname{F}}^2 \right]$ is expressed as
\begin{equation}\label{eq:exp_HHT}
    D \sum_{i,j=1}^N (D+1)(\tilde{\mA}^{2k})_{ij}^2 +  (\tilde{\mA}^{2k})_{ii}(\tilde{\mA}^{2k})_{jj},
\end{equation}
which increases monotonically with $k$.
\end{theorem}

In addition to the theorem where the assumption might be too simplified, we consider the case where the weights follow a Glorot initialization. We plot the feature correlation in Figure~\ref{fig:corr_plot}. The result shows a clear increase of feature interdependencies as GCN depth increases. These findings suggest that message passing GNNs such as GCN may not be well-suited for GDA tasks. 
\begin{figure}[t]
    \centering
    \includegraphics[width=0.4\textwidth]{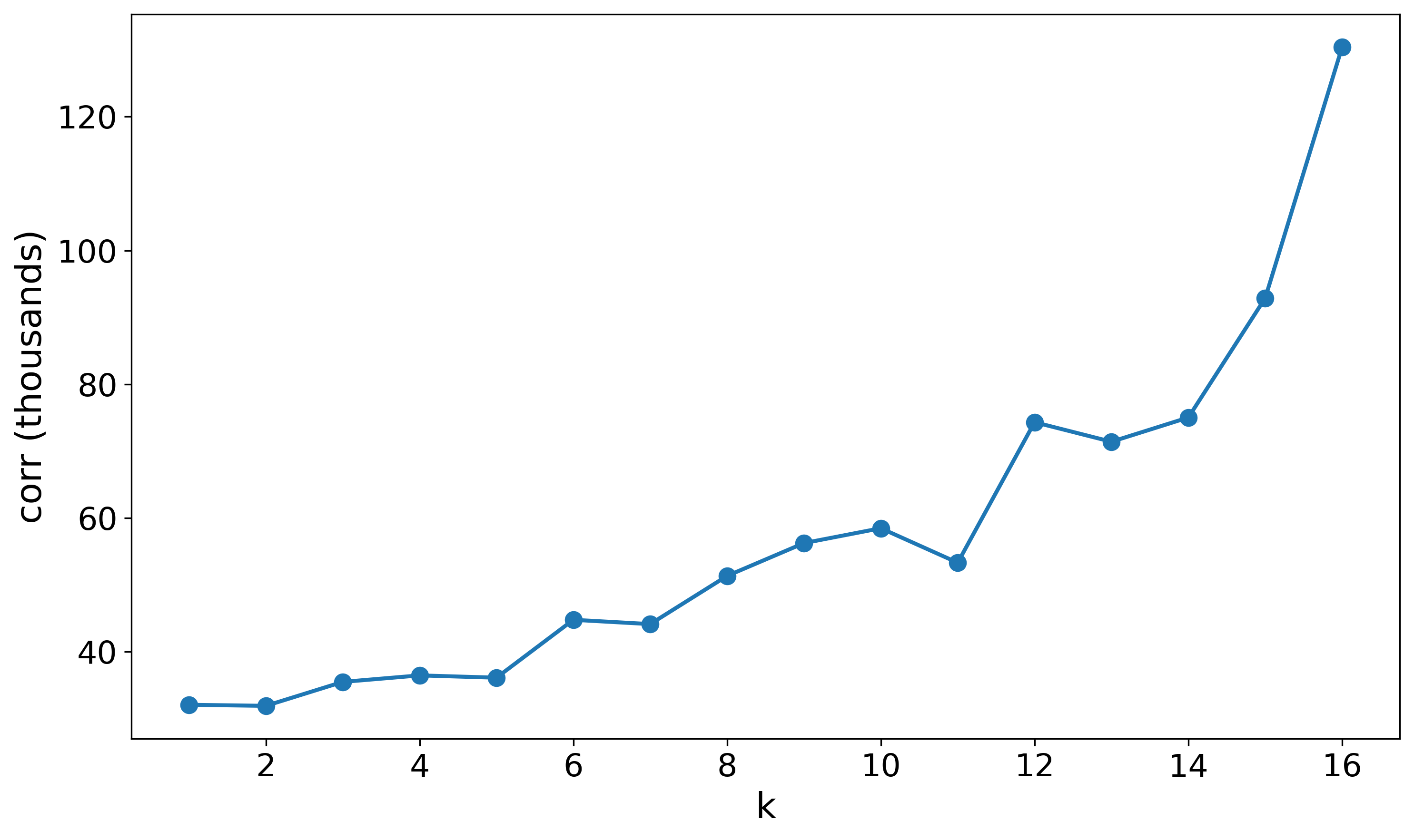}
    \caption{How the correlation $\E [ \|{\mH^{(k)} \mH^{(k) \T}}\|_{\operatorname{F}}^2 ]$ changes with the number of layers $k$.}
    \label{fig:corr_plot}
\end{figure}

We remark that node feature correlation is also related to oversmoothing, a well-known phenomenon of message passing GNN which is quantified using the Dirichlet Energy~\cite{cai2020note} and the Mean Average Distance (MAD)~\cite{chen2020measuring}. However, these metrics do not capture the specific effects of feature correlation as shown in our analysis.

\section{Improving GDA with Decorrelation}
Based on the above analysis, we focus on reducing feature correlation in GCN representations for GDA. We adopt the classical UDAGCN model~\cite{wu2020unsupervised} as a representative framework and replace its GCN layers with feature decorrelation operations. This modification can be naturally extended to other backbone architectures, as discussed in Section~\ref{subsec:decorrelation_generalizability}.

\paragraph{Decorrelated Feature Extraction}
As claimed in \cite{zhao2020pairnorm, chen2021graph, ma2021unified}, GNN layers can be regarded as gradient steps of graph signal denoising. To decorrelate the node representations, for both domains $r=s,t$, we consider the graph signal denoising  objective function regularized by a decorrelation term. Namely, for a given graph signal $\mX_r$, one seeks to find $\mH_r$ that minimizes
\begin{equation}\label{eq:decorr_opt}
    \frac{1}{2} \norm{\mH_r - \mX_r}_\F^2 + \frac{\lambda_1}{2} \tr (\mH_r^\T \tilde{\mL}_r \mH_r) + \frac{\lambda_2}{4} \norm{\mH_r \mH_r^\T - \mI}_\F^2,
\end{equation}
where $\tilde{\mL}_r = \mI - (\mD_r+\mI)^{-1/2} \tilde{\mA}_r (\mD_r+\mI)^{-1/2}$ is the normalized Laplacian used in GCN, with $\mD_r$ denoting the degree matrix in domain $r$. Here, the $\lambda_1$ term is a universal smoothing term in graph signal denoising \cite{ma2021unified}, whereas the $\lambda_2$ term encourages the rows -- that is, the node features -- of $\mH_r$ to be orthogonal. We remark that \citet{liu2023enhancing} also introduced a similar $\lambda_2$ term, but with $\mH^\T \mH$ instead of $\mH \mH^\T$ since their objective is to minimize the dimension-wise correlation instead of feature-wise correlation.  The gradient of \eqref{eq:decorr_opt} is
\begin{equation}
    G(\mH_r) = \mH_r - \mX_r + \lambda_1 \tilde{\mL}_r \mH_r + \lambda_2 (\mH_r \mH_r^\T - \mI) \mH_r.
\end{equation}
Consider $\mH_r^{(0)} = \mX_r$ to be the input node features. Let $\mH_r^{(l)}$ denote the features in the $l$-th layer of a GNN. The $(l+1)$-th layer then follows the gradient step accordingly:
\begin{equation}\label{eq:grad_des_simple}
    \mH_r^{(l+1)} = \mH_r^{(l)} - \gamma G(\mH_r^{(l)}).
\end{equation}

To further reduce correlation, we also apply the graph transformer layers by \citet{dwivedi2020generalization} which employ sparse attention mechanism and positional encoding. Specifically, the sparse attention block operates as: 
\begin{align}
    \operatorname{SA}(\mQ,\mK,\mV,\mA_r)=\operatorname{softmax} \left( \frac{(\mQ\mK^\T)\circ \mA_r}{\sqrt{\Tilde{D}}} \right) \mV,
\end{align}
where $\tilde{D}$ is the latent feature dimension; $\mQ, \mK, \mV \in \R^{N \times \tilde{D}}$ are the query, key and value matrices, shared in both source and target domains, and $\mA_r$ is the adjacency matrix in domain $r$, $r=s,t$. On the other hand, the positional encoding utilizes the two smallest non-trivial Laplacian eigenvectors of a node, denoted by $\vphi_r$, $r=s,t$. Given inputs $\mH_r \in \R^{N_r \times \tilde{D}}$ and $\vphi_r \in \R^{N_r \times k}$, the features with positional encoding are produced firstly:
\begin{equation}
    \mH_r^\dagger = \mH_r + \operatorname{Denselayer}(\vphi_r). \label{eq:hrhrpr}
\end{equation}
Then, each layer applies the following operations:
\begin{align}
    \mH_r^\ddagger &= \operatorname{BatchNorm}(\mH_r^\dagger+\operatorname{SA}(\mA_r,\mH_r^\dagger)),\\
    \mH_r^\wr &= \operatorname{ReLU}(\operatorname{Denselayer}(\mH_r^\ddagger)),\\
    \mZ_r &= \operatorname{BatchNorm}(\mH_r^\ddagger+\operatorname{Denselayer}(\mH_r^\wr)). \label{eq:zrhrhr}
\end{align}

For brevity, we name our model as DFT, short for \textbf{D}ecorrelated \textbf{F}eature Extraction and Graph \textbf{T}ransformer layers. The details of DFT architecture are presented in Appendix~\ref{app:network}. The time complexity of DFT is dominated by the complexity of the graph transformer layers, which is further dominated by matrix multiplication of size $N \times \tilde{D}$ and $\tilde{D} \times k$. Therefore, the time complexity is $O(N k \tilde{D})$. 

The overall objective of our model is as follows:
\begin{equation}
    \min_{\vtheta_\textup{feat},\vtheta_\textup{clf}} \max_{\vtheta_\textup{critic}}
    \left\{\gL_s+
    \lambda_t\gL_t+
    \lambda_\textup{critic} \left( \gL_\textup{critic}-
    \lambda_{\textup{gp}}\gL_{\textup{gp}} \right) \right\},
\end{equation}
where $\vtheta_\textup{feat}$, $\vtheta_\textup{clf}$ and $\vtheta_\textup{critic}$ summarize the trainable parameters in the feature extractor, the classifier and the domain critic, respectively. 
$\lambda_t$, $\lambda_\textup{critic}$ and $\lambda_{\textup{gp}}$ are hyperparameters.
$\gL_s$, $\gL_t$ represent the losses of classifiers in source and target domains, $\gL_\textup{critic}$ represents the loss of the critic and $\gL_{\textup{gp}}$ represents the gradient penalty term. The details of the loss terms are presented in Appendix~\ref{app:loss}. The hyperparameters are chosen to be $\displaystyle \lambda_t = \frac{\textup{current epoch}}{\textup{epochs} \times 100}$, $\lambda_\textup{critic} = 1$ and $\lambda_\textup{gp} = 10$. 
For clarity, we present the algorithm for training DFT as Algorithm~\ref{algorithm}. 
\begin{algorithm}[h]
    \caption{DFT training algorithm for unsupervised GDA in node classification}
    \label{algorithm}
    \begin{algorithmic}[1]
    \Require
    source graph $\gG_s$ with adjacency matrix $\mA_s$, node features $\mX_s$, node class labels $\mY_s$; target graph $\gG_t$ with adjacency matrix $\mA_t$, node features $\mX_t$; domain critic training step $n_\textup{critic}$; number of layers $L$; coefficients $\lambda_t$, $\lambda_\textup{critic}$, $\lambda_\textup{gp}$; learning rate $\eta$; initialized neural network parameters $\vtheta_\textup{feat}$, $\vtheta_\textup{clf}$, $\vtheta_\textup{critic}$.
    \Ensure
    updated neural network parameters $\vtheta_\textup{feat}$, $\vtheta_\textup{clf}$, $\vtheta_\textup{critic}$.
    \While{$\vtheta_\textup{feat}$, $\vtheta_\textup{clf}$, $\vtheta_\textup{critic}$ not converged and maximum iteration number not reached}
    \For{$r = s, t$}
    \For{$i = 1, \cdots, L$}
        \State $\mH_r^{(i)} \gets$ update using~\eqref{eq:grad_des_simple}
    \EndFor
    \State $\mH_r \gets \mH_r^{(L)}$ 
    \State $\mZ_r \gets$  update using~\eqref{eq:hrhrpr}--\eqref{eq:zrhrhr}
    \EndFor
    \For{$j=1, \cdots, n_\textup{critic}$}
    \State \hspace{-5pt} $\vtheta_\textup{critic} \gets \operatorname{Adam}\left(-\eta, \nabla_{\vtheta_\textup{critic}} (\gL_\textup{critic}-\lambda_\textup{gp}\gL_\textup{gp}) \right)$
    \EndFor
    \State $\vtheta_\textup{feat} \gets \operatorname{Adam}\left(\eta, \nabla_{\vtheta_\textup{feat}} (\gL_s+\lambda_t\gL_t+\lambda_\textup{critic}\gL_\textup{critic}) \right)$
    \State $\vtheta_\textup{clf} \gets \operatorname{Adam} \left( \eta, \nabla_{\vtheta_\textup{clf}} (\gL_s+\lambda_t\gL_t+\lambda_\textup{critic}\gL_\textup{critic}) \right)$
    \EndWhile
    \end{algorithmic}
\end{algorithm}

\section{Experiments}\label{sec:experiments}
\subsection{Datasets}\label{subsec:Datasets}
\begin{table}[b]
    \centering
    \caption{Statistics of datasets}
    \scalebox{0.9}{
    \label{tab:Table 2}
    \begin{tabular}{c|cccc}
        \toprule
        Dataset & \#Nodes & \#Edges & \#Attributes & \multicolumn{1}{c}{\#Labels} \\
        \midrule
        DBLPv7 & 5,484 & 8,130 & 6,775 & 5 \\
        Citationv1 & 8,935 & 15,113 & 6,775 & 5 \\
        ACMv9 & 9,360 & 15,602 & 6,775 & 5 \\
        Blog1 & 2,300 & 33,471 & 8,189 & 6 \\
        Blog2 & 2,896 & 53,836 & 8,189 & 6 \\
        Pubmed1 & 9858 & 21,818 & 500 & 3 \\
        Pubmed2 & 9859 & 22,538 & 500 & 3 \\
        \bottomrule
    \end{tabular}
    }
\end{table}

We consider the following cross-network datasets in our experiments. The statistics of the datasets are shown in Table~\ref{tab:Table 2}. An additional dataset is considered in Appendix~\ref{app:largedataset}.

\emph{Citation}. 
The Citation dataset contains three domains, namely DBLP, Citation, and ACM, constructed by \citet{shen2020network} based on datasets in \cite{tang2008arnetminer}. To be consistent with \cite{dai2022graph}, we use the versions DBLPv7, Citationv1, and ACMv9. Spanning the temporal domain, these data encompass papers published from 2004 to 2010, modeled as nodes in directed graphs, where each edge represents a citation between two papers. Each node is labeled into one of five categories based on its topic. The node attributes are bag-of-words vectors representing keywords extracted from their titles. It is therefore a scenario where covarite shift assumption naturally holds when we consider domain adaptation among these datasets.

\emph{Blog}. 
The Blog dataset contains two domains, namely Blog1 and Blog2. Blog1 and Blog2 are two distinct social networks derived from the BlogCatalog dataset \cite{li2015unsupervised}, where each node symbolizes a blogger, and each edge signifies the friendship between two bloggers. The attributes of each node consist of keywords extracted from the blogger's self-description. Additionally, each node is assigned a label that identifies the group it belongs to. We use the version processed by \citet{shen2020adversarial}.

\emph{Pubmed}. 
To show the performance on larger datasets, we have randomly divided the Pubmed dataset~\cite{sen2008collective} into two smaller datasets as source and target graphs.

We consider unsupervised node classification using eight transfer learning tasks, namely D $\rightarrow$ C, C $\rightarrow$ D, A $\rightarrow$ C, C $\rightarrow$ A, D $\rightarrow$ A, A $\rightarrow$ D, Blog1 $\rightarrow$ Blog2, Blog2 $\rightarrow$ Blog1, Pubmed1 $\rightarrow$ Pubmed2, Pubmed2 $\rightarrow$ Pubmed1, where D, C, A refer to the datasets DBLPv7, Citationv1 and ACMv9, respectively.
All labeled data are available in the source domains, and all instances but not their labels are available in the target domains.

\subsection{Implementation Details}\label{Implementation Details}
We implement our method using PyTorch on a Linux server with RTX 3090 GPUs (24GB) and 14 vCPU Intel(R) Xeon(R) Gold 6330 CPU. Each implementation runs on a single GPU. Our architecture contains three layers of decorrelated GCN and four layers of graph transformers. We use Adam \cite{kingma2014adam} for optimization and the learning rate is set to 0.003. The dropout rate is chosen from \{0.1, 0.2, 0.5\}. For the hyperparameters in decorrelated feature extraction, we choose $\lambda_1$ from \{40, 60, 80, 100, 120\}, $\lambda_2$ from \{0.0001, 0.0003, 0.0005, 0.0008, 0.001\}, and $\gamma$ from \{0.001, 0.003, 0.005, 0.008, 0.01\}. In most cases, setting $\lambda_1$, $\lambda_2$ and $\gamma$ to 100, 0.001, and 0.01, respectively, yields optimal results. Each implementation runs 500 epochs. 

\subsection{Baseline Methods}\label{app:Baselines}
We compare DFT with the following GDA approaches. We utilize respective official codes of the baseline methods for their implementations. 
\begin{itemize}[listparindent=0pt, leftmargin=*]
\item\textbf{DANN}~\cite{ganin2016domain}
employs an architecture comprising a feature extractor, taken to be GCN in our implementation, and a gradient reversal layer. 
\item\textbf{AdaGCN}~\cite{dai2022graph}
employs a GCN-based feature extractor and enhances its domain adaptation capabilities by adversarial training.
\item\textbf{UDAGCN}~\cite{wu2020unsupervised}
develops upon the DANN framework by aggregating information extracted by both the adjacency matrix and the PPMI matrix using an attention mechanism. 
\item\textbf{StruRW}~\cite{liu2023structural}
computes edge probabilities between different classes on the target graph, which guide the selection of neighbors for GNN computations on the source graph.
\item\textbf{SpecReg}~\cite{you2023graph} employs spectral properties including spectral smoothness and maximum frequency response to design regularization terms. This regularization scheme significantly bolsters GNN transferability.
\item\textbf{GRADE}~\cite{wu2023noniid}
introduces the concept of graph subtree discrepancy. By quantifying the shift in graph distributions between source and target networks, GRADE optimizes to minimize this discrepancy, thereby enhancing cross-network transfer learning.
\item \textbf{A2GNN}~\cite{liu2024rethinking} investigates the underlying generalization capability of GNN and uses an asymmetric architecture for source and target graphs.
\item \textbf{PairAlign}~\cite{liu2024pairwise} designs specific edge weights and label weights to align the shifts in conditional feature distribution and label distribution.
\end{itemize}

\subsection{Experimental Results}\label{Cross-Domain CLassification Results}
We report the mean of F1-scores (both micro and macro) for the target graph from three independent implementations, and record the standard deviation as well. Table~\ref{tab:Table 3} showcases the F1-scores. We also report the runtime for each epoch in Appendix~\ref{app:runtime}.

Our method consistently demonstrates superior performance, frequently achieving the highest F1 scores and occasionally the second-best. In several cases, the performance margin is substantial. Such consistent strong performance highlights the effectiveness of DFT in addressing the challenges of GDA tasks. We also notice that the non-graph method of DANN occasionally outperforms specialized methods designed for conditional shift. This indicates that the underlying assumptions of conditional shift may not apply to all graph datasets in a general GDA scenario. Finally, DFT consistently and significantly outperforms the vanilla UDAGCN, which highlights the effectiveness of the decorrelation layers.

\begin{table*}[t]
    \centering
    \caption{F1-score results (\%) for comparison with baseline methods (best in \textbf{bold}; second-best \underline{underlined})}
    \label{tab:Table 3}
    \scalebox{0.83}{
    \begin{tabular}{ccccccccccc}
        \toprule
         &F1-score &DANN & AdaGCN & UDAGCN & StruRW & SpecReg & GRADE &A2GNN &PairAlign& DFT (ours)\\
        \midrule
        \multirow{2}{*}{D $\rightarrow$ C}& micro &$50.76\pm 0.03$ & $65.44\pm 2.03$ & $52.90\pm 0.69$ &$58.92\pm 1.30$ & $64.77\pm 0.20$ & $63.32\pm 2.44$ &$\underline{78.36}\pm 0.64$ &$62.20\pm 3.12$& $\boldsymbol{78.45}\pm 0.19$\\
        & macro &$45.35\pm 0.07$ & $61.70\pm 1.68$ & $40.25\pm 0.59$ &$54.09\pm 0.89$ & $60.46\pm 0.30$ & $54.18\pm 4.64$  &$\underline{75.44}\pm 0.70$ &$53.70\pm 1.29$& $\boldsymbol{75.54}\pm 0.91$\\
        \midrule
        \multirow{2}{*}{C $\rightarrow$ D}& micro &$63.93\pm 0.00$ & $68.35\pm 5.58$ & $69.28\pm 0.80$ &$61.06\pm 1.39$ & $67.84\pm 1.99$ & $69.13\pm 0.76$  &$\underline{73.57}\pm 0.64$& $66.24\pm 0.82$&$\boldsymbol{75.77}\pm 0.30$ \\
        & macro &$52.36\pm 0.00$ & $62.51\pm 10.78$ & $64.04\pm 0.18$ &$56.16\pm 1.92$ & $63.93\pm 0.21$ & $66.02\pm 0.74$  &$\underline{71.12}\pm 1.36$ &$57.00\pm 3.40$&$\boldsymbol{73.05}\pm 0.09$ \\
        \midrule
        \multirow{2}{*}{A $\rightarrow$ C} & micro &$52.13\pm 0.00$ & $60.60\pm 7.02$ & $55.03\pm 0.13$ &$58.92\pm 1.30$ & $65.87\pm 0.08$ & $69.12\pm 0.29$ &$\underline{77.49}\pm 0.62$ &$64.09\pm 1.83$&$\boldsymbol{77.69}\pm 0.14$\\
        & macro &$39.52\pm 0.00$ & $53.15\pm 10.46$ & $42.93\pm 0.11$ &$54.34\pm 1.86$ & $62.32\pm 0.11$ & $65.09\pm 3.25$  &$\underline{74.98}\pm 0.87$ &$57.46\pm 1.51$&$\boldsymbol{75.03}\pm 0.39$ \\
        \midrule
        \multirow{2}{*}{C $\rightarrow$ A}& micro &$55.47\pm 0.00$ & $64.79\pm 2.30$ & $62.79\pm 0.08$ &$54.89\pm 0.82$ & $62.77\pm 0.08$ & $63.07\pm 1.76$ &$\boldsymbol{74.87}\pm 0.30$ &$59.16\pm 0.76$&$\underline{69.31}\pm 0.29$\\
        & macro &$46.77\pm 0.00$ & $63.93\pm 2.36$ & $52.82\pm 0.08$ &$51.72\pm 2.62$ & $62.54\pm 0.14$ & $62.98\pm 1.89$  &$\boldsymbol{76.25}\pm 0.30$&$54.88\pm 2.66$ &$\underline{69.21}\pm 0.09$ \\
        \midrule
        \multirow{2}{*}{D $\rightarrow$ A} & micro &$49.08\pm 0.00$ & $59.43\pm 2.05$ & $52.51\pm 0.14$ & $51.64\pm 1.21$& $54.41\pm 0.31$ & $58.18\pm 1.26$  &$\boldsymbol{71.26}\pm 0.08$&$55.02\pm 0.64$ & $\underline{70.65}\pm 0.19$\\
        & macro &$41.47\pm 0.02$ & $59.09\pm 1.71$ & $43.59\pm 0.19$ &$47.38\pm 1.59$ & $52.26\pm 4.76$ & $53.51\pm 3.91$ & $\boldsymbol{72.90}\pm 0.03$&$46.63\pm 0.47$ &$\underline{71.07}\pm 0.31$ \\
        \midrule
        \multirow{2}{*}{A $\rightarrow$ D} & micro &$56.12\pm 0.00$ & $63.86\pm 5.42$ & $61.22\pm 1.77$ &$57.02\pm 1.67$ & $65.16\pm 0.07$ & $\underline{67.16}\pm 0.55$ &$66.71\pm 0.56$&$63.43\pm 1.63$& $\boldsymbol{74.84}\pm 0.38$\\
        & macro &$43.30\pm 0.00$ & $59.27\pm 7.58$ & $54.14\pm 4.75$ &$51.07\pm 2.09$ & $61.22\pm 0.12$ & $61.57\pm 3.18$  &$\underline{63.08}\pm 1.11$&$52.66\pm 1.51$& $\boldsymbol{72.85}\pm 0.89$ \\
        \midrule
        \midrule
        \multirow{2}{*}{Blog1 $\rightarrow$ Blog2} & micro &$24.94\pm 0.24$ &$\underline{34.55}\pm 5.93$ &$25.06\pm 0.05$  &$18.89\pm 3.18$ &$20.03\pm 1.26$  &$25.60\pm 4.41$ &$25.02\pm 1.82$&$27.45\pm 2.76$&$\boldsymbol{39.43}\pm 1.39$ \\
        & macro &$17.87\pm 0.25$ &$\underline{31.57}\pm 6.38$ &$18.98\pm 0.05$  &$7.18\pm 4.54$  &$10.12\pm 2.01$  &$14.98\pm 4.14$ &$15.14\pm 4.05$ &$22.98\pm 4.74$&$\boldsymbol{37.12}\pm 1.50$  \\
        \midrule
        \multirow{2}{*}{Blog2 $\rightarrow$ Blog1} & micro &$23.96\pm 0.35$ &$\underline{29.15}\pm 3.69$ &$19.50\pm 0.02$  &$17.54\pm 0.94$ &$20.27\pm 0.76$  &$23.08\pm 2.07$ &$24.93\pm 0.84$&$27.64\pm 2.26$&$\boldsymbol{43.87}\pm 0.65$ \\
        & macro &$17.80\pm 1.74$ &$\underline{27.50}\pm 6.27$ &$9.97\pm 0.00$  &$6.36\pm 2.90$  &$10.60\pm 1.90$  &$13.86\pm 0.50$ &$12.01\pm 0.54$ &$23.52\pm 3.17$&$\boldsymbol{38.86}\pm 2.06$  \\ 
        \midrule
        \midrule
        \multirow{2}{*}{Pubmed1 $\rightarrow$ 2} & micro &$85.31\pm 0.05$ &$\underline{85.53}\pm 0.13$ &$84.52\pm 0.01$  &$69.08\pm 0.70$ &$85.25\pm 0.05$  &$84.95\pm 0.02$ &$81.24\pm 0.25$&$84.41\pm 0.65$&$\boldsymbol{88.73}\pm 0.19$ \\
        & macro &$84.75\pm 0.05$ &$\underline{84.84}\pm 0.15$ &$83.85\pm 0.16$  &$67.53\pm 0.82$  &$84.64\pm 0.02$  &$84.23\pm 0.05$ &$80.01\pm 0.05$ &$84.24\pm 0.56$&$\boldsymbol{88.51}\pm 0.18$  \\ 
        \midrule
        \multirow{2}{*}{Pubmed2 $\rightarrow$ 1} & micro &$86.03\pm 0.01$ &$80.58\pm 0.32$ &$85.23\pm 0.04$  &$69.13\pm 0.72$ &$\underline{86.24}\pm 0.09$  &$85.81\pm 0.08$ &$80.58\pm 0.32$&$83.98\pm 0.39$&$\boldsymbol{89.20}\pm 0.18$ \\
        & macro &$85.77\pm 0.01$ &$79.46\pm 0.53$ &$84.80\pm 0.02$  &$68.17\pm 0.88$  &$\underline{86.90}\pm 0.10$  &$85.37\pm 0.10$ &$79.46\pm 0.53$ &$83.94\pm 0.72$&$\boldsymbol{89.15}\pm 0.10$  \\
        \bottomrule
    \end{tabular}
    }
\end{table*}

\begin{table*}[t]
    \centering
    \caption{F1-score results (\%) for ablation studies (best in \textbf{bold})}
    
    \scalebox{0.9}{
    \label{tab:Table 4}
    \begin{tabular}{cccccccccc}
        \toprule
         &F1-score & DFT-GCN & DFT-noT &DFT-pureT& DFT($\lambda_1=0$) & DFT($\lambda_2=0$) &DFT-mmd&DFT-DropEdge &DFT (ours)\\
        \midrule
        \multirow{2}{*}{D $\rightarrow$ C}& micro &$75.95\pm 0.75$&$65.05\pm 0.20$&$65.88\pm 0.82$&$73.06\pm 0.05$ &$76.51\pm 0.04$ & $76.11\pm 0.25$&$75.47\pm 0.44$&$\boldsymbol{78.45}\pm 0.19$\\
        & macro &$71.78\pm 0.93$&$60.28\pm 0.26$&$63.13\pm 0.64$ &$67.93\pm 0.38$&$71.59\pm 0.33$ &$73.58\pm 0.27 $& $72.30 \pm 0.43$&$\boldsymbol{75.54}\pm 0.91$\\
        \midrule
        \multirow{2}{*}{C $\rightarrow$ D}& micro &$75.28\pm 0.38$& $71.56\pm 0.07$&$74.15\pm 0.64$ &$71.90\pm 0.43$ &$75.20\pm 0.15$ & $71.87\pm 0.30$&$73.51\pm 0.95$&$\boldsymbol{75.77}\pm 0.30$ \\
        & macro &$72.46\pm 0.53$&$69.16\pm 0.14$&$71.62\pm 0.97$ &$67.53\pm 0.42$ &$72.53\pm 0.60$ &$68.44\pm 0.30$& $71.56\pm 1.03$&$\boldsymbol{73.05}\pm 0.09$ \\
        \midrule
        \multirow{2}{*}{A $\rightarrow$ C}& micro &$76.64\pm 0.12$&$67.17\pm 0.05$&$74.83\pm 1.17$ &$66.89\pm 1.13$ &$76.76\pm 0.06$ & $73.10\pm 0.41$&$76.74\pm 0.50$&$\boldsymbol{77.69}\pm 0.14$\\
        & macro &$73.99\pm 0.05$&$63.03\pm 0.32$&$72.67\pm 1.22$ &$63.81\pm 1.25$ &$73.70\pm 0.36$ &$70.46\pm 0.36$& $74.43\pm 0.65$&$\boldsymbol{75.03}\pm 0.39$ \\
        \midrule
        \multirow{2}{*}{C $\rightarrow$ A} & micro &$66.88\pm 0.60$&$58.89\pm 0.72$&$65.73\pm 0.28$ &$55.65\pm 0.92$ &$63.70\pm 0.37$ &$65.08\pm 0.18$&$69.79\pm 0.76$&$\boldsymbol{69.31}\pm 0.29$\\
        & macro &$66.67\pm 0.96$&$58.08\pm 0.93$& $65.69\pm 0.42$&$53.88\pm 1.40$ &$63.80\pm 0.26$ &$64.85\pm 0.47$& $70.09\pm 0.99$&$\boldsymbol{69.21}\pm 0.09$ \\
        \midrule
        \multirow{2}{*}{D $\rightarrow$ A}& micro &$59.64\pm 0.93$&$60.67\pm 0.18$&$61.96\pm 0.93$ &$60.37\pm 0.23$ &$66.26\pm 1.43$ &$60.66\pm 0.15$&$66.37\pm 0.35$ & $\boldsymbol{70.65}\pm 0.19$\\
        & macro &$53.96\pm 1.18$&$58.26\pm 0.44$&$62.20\pm 1.07$ &$56.35\pm 0.75$ &$63.67\pm 4.21$ &$53.48\pm 0.26$& $67.00\pm 0.30$&$\boldsymbol{71.07}\pm 0.31$ \\
        \midrule
        \multirow{2}{*}{A $\rightarrow$ D}& micro & $73.56\pm 1.48$&$70.26\pm 0.46$&$72.50\pm 0.21$ &$63.81\pm 0.48$ &$74.38\pm 0.32$ & $69.06\pm 0.15$&$71.50\pm 0.92$& $\boldsymbol{74.84}\pm 0.38$\\
        & macro &$70.04\pm 2.84$&$66.87\pm 0.49$&$69.85\pm 0.64$ &$61.15\pm 0.61$ &$72.12\pm 0.25$ &$65.61\pm 0.16$&$69.04\pm 0.87$& $\boldsymbol{72.85}\pm 0.89$ \\
        \bottomrule
    \end{tabular}
    }
\end{table*}

\subsection{Ablation Studies}
To further validate the effectiveness of the decorrelation and transformer layers, we perform the following ablation studies. We consider variants of DFT where these components are omitted, yielding the following configurations:
\begin{itemize}[listparindent=0pt, leftmargin=*]
    \item\textbf{DFT-GCN} uses standard GCN layers and graph transformers.
    \item\textbf{DFT-noT} uses decorrelated GCN but no transformer layers.
    \item\textbf{DFT-pureT} replaces graph transformer layers by standard transformer self-attention layers.
    \item\textbf{DFT-mmd} replaces adversarial loss by the maximum mean discrepancy (MMD) loss~\cite{shen2020network}.
    \item\textbf{DFT-DropEdge} replaces decorrelated GCN by DropEdge~\cite{rong2019dropedge}, a method for reducing node connectivity.
\end{itemize}

We also conduct experiments on DFT when $\lambda_1=0$ and $\lambda_2=0$, i.e., the ablation studies of~\eqref{eq:decorr_opt}. The results for the ablation studies are shown in Table~\ref{tab:Table 4}. Clearly, removing any component results in a deterioration of results. We conclude that the effectiveness of DFT is attributed to the combination of components for alleviating local interdependencies. Nevertheless, even partial inclusion of these components improves performance over baseline UDAGCN model. To explain the contributions, we calculate the Shapley values~\cite{lundberg2017unified} of decorrelation and transformer, which are 34.36 and 40.09, respectively (46.19\% and 53.81\% after normalization), which suggests that these two components have similar contributions.

\subsection{Feature Visualization and Analysis}
We visualize the representations learned using our approach and baseline methods by projecting them onto a two-dimensional space using the t-distributed stochastic neighbor embedding (t-SNE) \cite{van2008visualizing}.

\begin{figure*}[t]
  \centering
  \subfigure[DANN]{
    \label{fig:subfig:b}
    \includegraphics[width=0.3\textwidth]{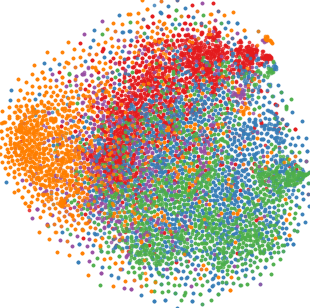}
  }
  \hfill
  \subfigure[UDAGCN]{
    \label{fig:subfig:c}
    \includegraphics[width=0.3\textwidth]{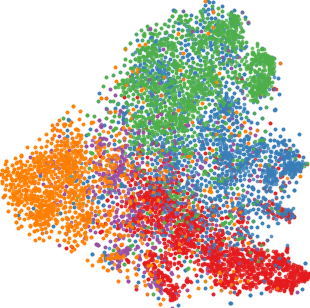}
  }
  \hfill
  \subfigure[AdaGCN]{
    \label{fig:subfig:d}
    \includegraphics[width=0.3\textwidth]{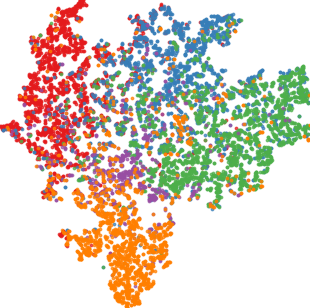}
  }\\
  \subfigure[GRADE]{
    \label{fig:subfig:e}
    \includegraphics[width=0.3\textwidth]{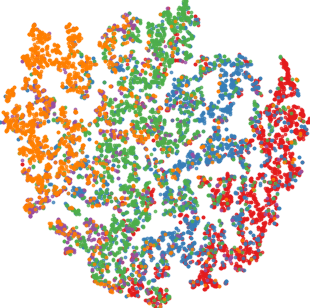}
  }
  \hfill
  \subfigure[StruRW]{
    \label{fig:subfig:f}
    \includegraphics[width=0.3\textwidth]{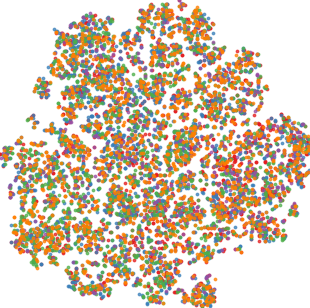}
  }
  \hfill
  \subfigure[SpecReg]{
    \label{fig:subfig:g}
    \includegraphics[width=0.3\textwidth]{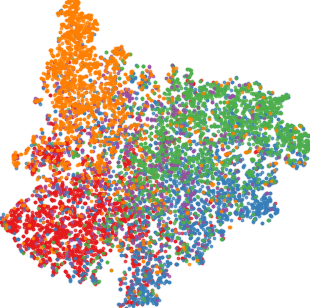}
  }\\
  \subfigure[A2GNN]{
    \label{fig:subfig:h}
    \includegraphics[width=0.3\textwidth]{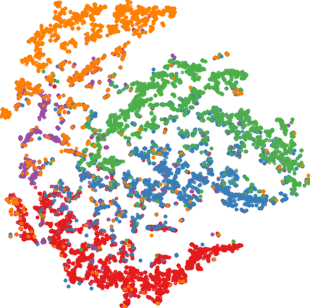}
  }
  \hfill
  \subfigure[PairAlign]{
    \label{fig:subfig:i}
    \includegraphics[width=0.3\textwidth]{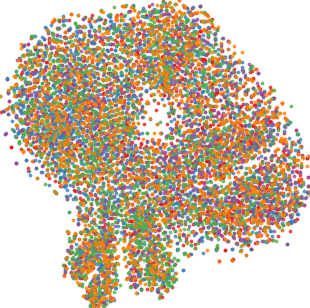}
  }
  \hfill
  \subfigure[DFT]{
    \label{fig:subfig:a}
    \includegraphics[width=0.3\textwidth]{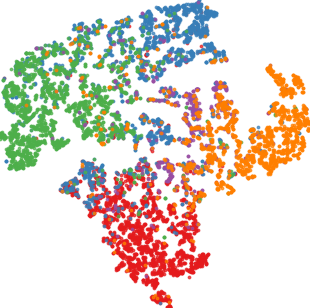}
  }
  \caption{Visualization of the representations learned in all methods (D $\to$ C).}
  \label{fig:dblp_citation}
\end{figure*}

\begin{figure*}[t]
  \centering
  \subfigure[$\gamma$ vs $\lambda_1$]{
    \label{fig:gamma_lambda_1}
    \includegraphics[width=0.32\textwidth]{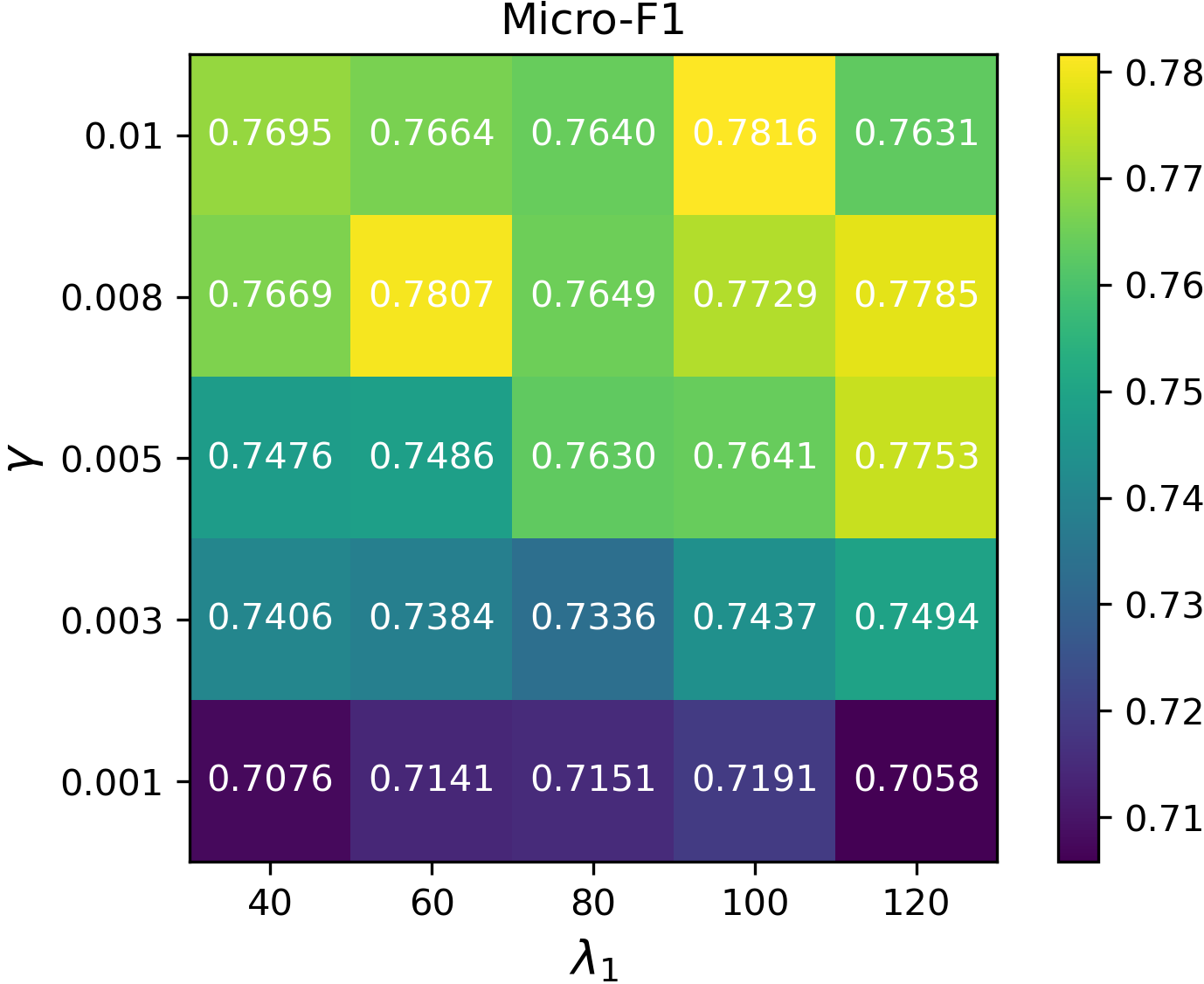}
  }
  \subfigure[$\gamma$ vs $\lambda_2$]{
    \label{fig:gamma_lambda_2}
    \includegraphics[width=0.32\textwidth]{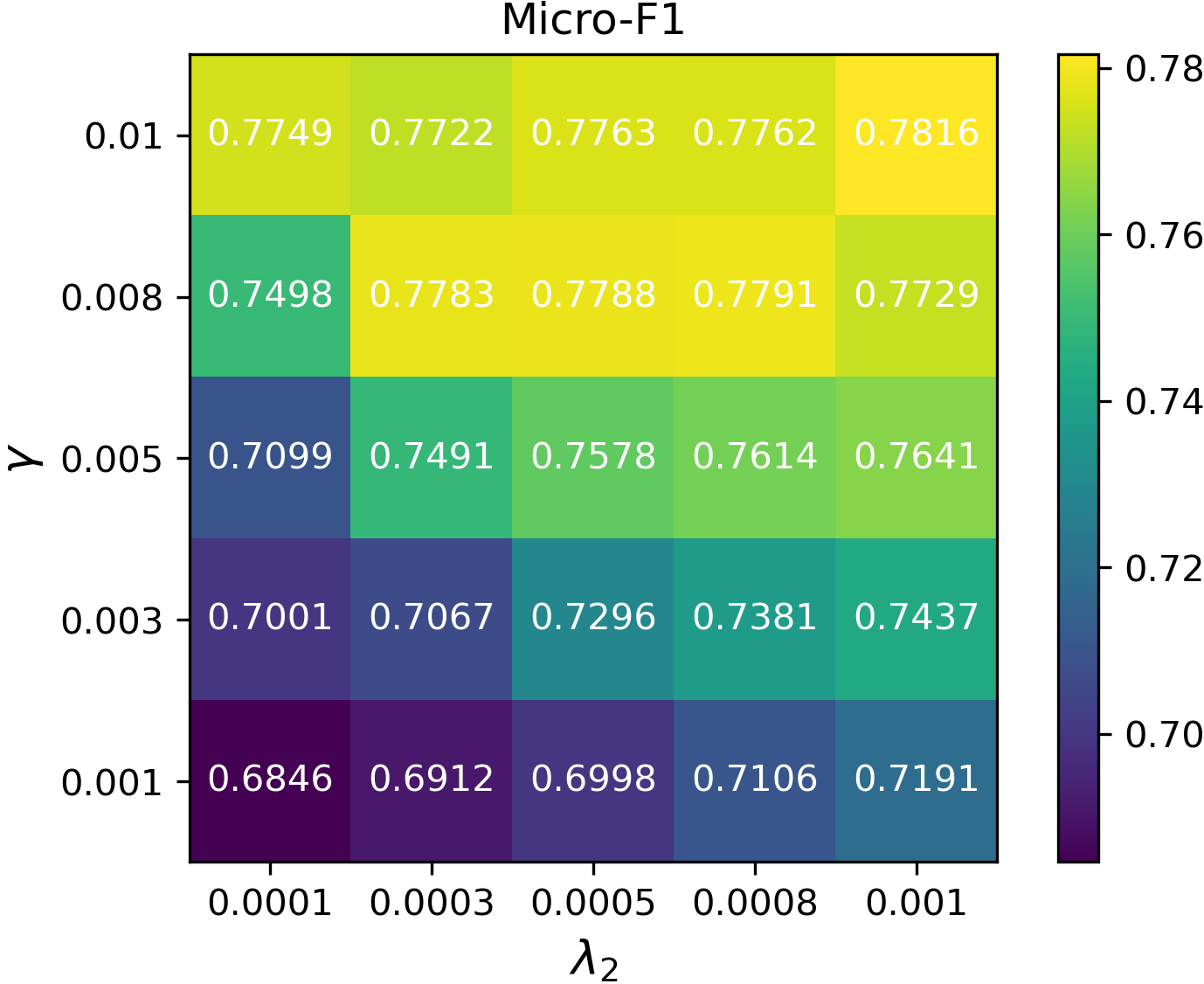}
  }
  \subfigure[$\lambda_1$ vs $\lambda_2$]{
    \label{fig:lambda_1_lambda_2}
    \includegraphics[width=0.31\textwidth]{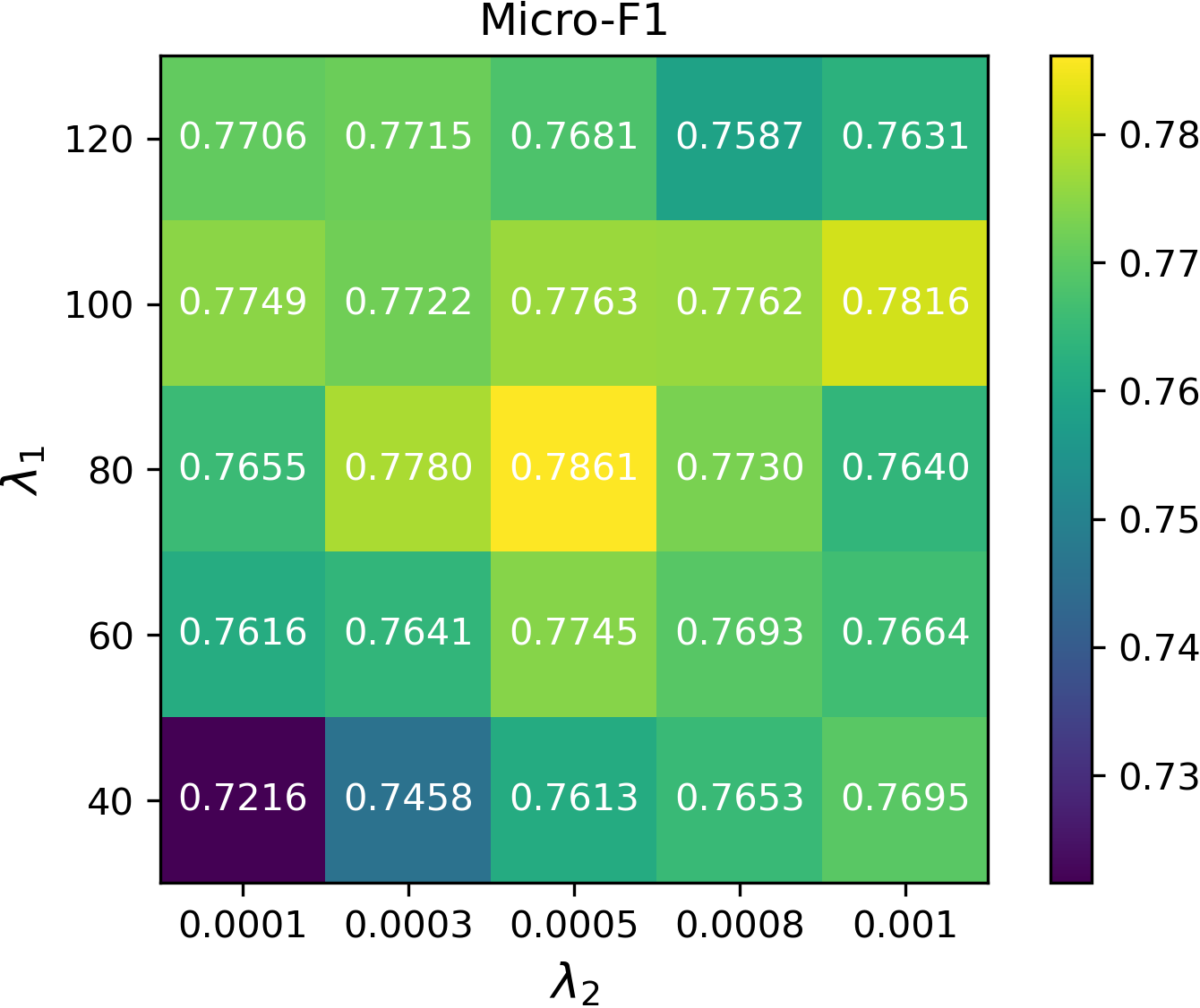}
  }
  \caption{Sensitivity analysis. Each figure shows the change of Micro-F1 scores with a pair of hyperparamaters.}
  \label{fig:all_gamma_lambda}
\end{figure*}

\begin{table*}[t]
    \centering
    \caption{ICDR(\%) ($\downarrow$) for D $\rightarrow$ C and A $\rightarrow$ C (best in \textbf{bold}; second-best \underline{underlined})}
    \scalebox{0.9}{
    \label{tab:ICDR}
    \begin{tabular}{cccccccccc}
        \toprule
          & DANN & UDAGCN & AdaGCN & StruRW & GRADE & SpecReg & A2GNN &PairAlign& DFT (ours)\\
        \midrule
        D $\rightarrow$ C & $94.63\pm 0.01$ & $81.09\pm 0.06$ & $72.95\pm 2.42$ &$99.73\pm 0.02$ & $80.58\pm 2.81$ & $83.42\pm 0.07$ &$\underline{71.12}\pm 0.12$ &$99.88\pm 0.13$& $\boldsymbol{59.43}\pm 0.33$ \\
        A $\rightarrow$ C & $88.52\pm 0.00$ & $81.96\pm 0.32$ & $72.31\pm 2.02$ &$99.58\pm 0.11$ & $84.31\pm 2.07$ & $79.68\pm 0.10$ & $\underline{71.89}\pm 0.16$&$98.93\pm 0.19$&$\boldsymbol{62.25}\pm 0.84$ \\
        \bottomrule
    \end{tabular}}
\end{table*}

\begin{table*}[t]
    \centering
    \caption{Micro-F1 scores (\%) on AdaGCN with and without the proposed decorrelation mechanism. The results are averaged over three runs (best in \textbf{bold}).}
    \scalebox{0.9}{
    \label{tab:adagcn_comparison}
    \begin{tabular}{ccccccc}
        \toprule
        Method & D $\rightarrow$ C & C $\rightarrow$ D & A $\rightarrow$ C & C $\rightarrow$ A & D $\rightarrow$ A & A $\rightarrow$ D \\
        \midrule
        AdaGCN & 65.44 & 68.35 & 60.60 & 64.79 & 59.43 & 63.86 \\
        AdaGCN w/ Decorrelation & \textbf{67.61} & \textbf{70.93} & \textbf{69.22} & \textbf{68.21} & \textbf{68.99} & \textbf{70.38} \\
        \bottomrule
    \end{tabular}}
\end{table*}

We perform case study into the two tasks D $\rightarrow$ C and A $\rightarrow$ C. Figure~\ref{fig:dblp_citation} illustrate the results obtained using all the baseline approaches for D $\rightarrow$ C. We additionally present A $\rightarrow$ C in Appendix~\ref{app:atoc}. Evidently from the figures, our proposed method achieves clear margins between classes. Only the state-of-the-art model of A2GNN has a similar qualitative separation level. This demonstrates that DFT is capable of generating a meaningful representations. 

In addition to 2-D visualization, we also quantitatively evaluate how representations within the same class are more closely clustered. To this end, we propose a metric called the intra-class distance ratio (ICDR) as follows. For a pair of nodes $i$ and $j$, and their representations $\vz_i$ and $\vz_j$, we define 
\begin{equation}
    \gD_\textup{intra}(i,j)=\norm{\vz_i - \vz_j}_2 \cdot \vone({y_i=y_j}),
\end{equation}
\begin{equation}
    \gD_\textup{inter}(i,j)=\norm{\vz_i - \vz_j}_2 \cdot \vone({y_i \neq y_j}),
\end{equation}
where $\vone$ is the indicator function. Next, we calculate the average distances within the same class and between different classes. Let $\abs{\cdot}$ denotes set cardinality. We define
\begin{equation}
    \overline{\gD}_\textup{intra} = \frac{\sum_{i=1}^{N_t} \sum_{j=i+1}^{N_t} \gD_\textup{intra}(i,j)}{\abs{\{(i,j): i < j, y_i=y_j\}}};
\end{equation}
\begin{equation}
    \overline{\gD}_\textup{inter} = \frac{\sum_{i=1}^{N_t} \sum_{j=i+1}^{N_t} \gD_\textup{inter}(i,j)}{\abs{\{(i,j): i < j, y_i \neq y_j\}}}.
\end{equation}
Finally, the ICDR can be calculated by taking their ratio as follows:
\begin{equation}
    \textup{ICDR} = \frac{\overline{\gD}_\textup{intra}}{\overline{\gD}_\textup{intra}+\overline{\gD}_\textup{inter}}.
\end{equation}

We present the ICDR results corresponding to the visualization in Table~\ref{tab:ICDR}. Notably, compared with all the baseline methods, DFT achieves a significantly lower ICDR, indicating that representations within the same class are much closer compared to those with different labels. This confirms the successful representation learning of our approach. We also present a different metric in Appendix~\ref{app:silhouette}.

\subsection{Parameter Sensitivity}
We perform sensitivity analysis of DFT on $\lambda_1$, $\lambda_2$ and $\gamma$. The experiments are conducted on $D\rightarrow C$, where we report Micro-F1 scores with $\gamma$, $\lambda_1$ and $\lambda_2$ from \{0.001, 0.003, 0.005, 0.008, 0.01\}, \{40, 60, 80, 100, 120\} and \{0.0001, 0.0003, 0.0005, 0.0008, 0.001\}, respectively. In Figure~\ref{fig:gamma_lambda_1}, we set $\lambda_2$ to 0.001 and report scores for various values of $\gamma$ and $\lambda_1$. In Figure~\ref{fig:gamma_lambda_2}, we set $\lambda_1$ to 0.01 and report scores for various values of $\gamma$ and $\lambda_2$. In Figure~\ref{fig:lambda_1_lambda_2}, we set $\gamma$ to 100 and report scores for various values of $\lambda_1$ and $\lambda_2$.

\subsection{Generalizability of Decorrelated GCN}
\label{subsec:decorrelation_generalizability}
To validate that the contribution of our proposed decorrelated GCN is robust and not limited to our specific choice of backbone framework, we conduct an additional experiment using AdaGCN~\citep{dai2022graph}. In this experiment, we replace the standard GCN layers within the AdaGCN framework with our proposed decorrelated GCN layers, while keeping all other components and hyperparameters identical to the original implementation. The comparative results are reported in Table~\ref{tab:adagcn_comparison}. The integration of decorrelated GCN layers yields consistent performance improvements across all six transfer tasks compared to the original AdaGCN. These findings demonstrate that the proposed decorrelation objective effectively enhances the representation learning capability of GCN-based models, regardless of the underlying backbone architecture.

\section{Conclusion}\label{sec:conclusion}
We have proposed improving GDA by alleviating interdependencies of node representations. Specifically, we have addressed the research questions: \textbf{RA1.} Observation of conditional shifts reveals the interdependencies among node features. \textbf{RA2.} The generalization performance can be improved by reducing interdependencies of node representations. \textbf{RA3.} Message-passing propagation may hinder node interdependencies. We have designed specific components aimed at reducing node dependencies, which have demonstrated strong performance on real-world GDA datasets.

There are several directions for future work. First, we will explore decorrelation approaches for graph-level GDA tasks. Additionally, we will investigate in conditional shift in other types of graph data such as dynamic graphs. Finally, we will extend the study to broader types of domains beyond graphs where data distributions are naturally non-i.i.d.

\newpage
\bibliographystyle{ACM-Reference-Format}
\balance
\bibliography{reference}

\appendix
\section*{Appendix}
\allowdisplaybreaks
\section{Additional Generalization Analysis}\label{app:generalization}
In addition to the markov chain considered in the main text, we consider modeling interdependencies of node representations using dependency graphs~\cite{janson2004large}, where dependency relations are confined to adjacent nodes. In our context, we call a graph $\gG_\textup{dep}$ with node set $\gV$ a dependency graph for $\{\rvx_s(v)\}_{v \in \gV}$ if $\{\rvx_s(v)\}_{v \in \gV_1}$ and $\{\rvx_s(v)\}_{v \in \gV_2}$ are independent whenever $\gV_1 \subset \gV$ and $\gV_2 \subset \gV$ are not adjacent in $\gG_\textup{dep}$. We assume that the dependency graph of node features coincides with the underlying graph of the graph dataset. 

Each dependency graph is associated with a metric named forest complexity, denoted as $\Lambda(\gG)$. It measures how close a graph is to a forest by determining the minimum number of node merges required to eliminate all cycles. For the definition and a detailed discussion, we refer to \cite{zhang2019mcdiarmid}. 
We have the following generalization bound, of which the proof is presented in Appendix~\ref{app:proof_inequality}.
\begin{theorem}\label{thm:dep_graph}
Suppose the hypothesis class $\gH$ has a VC-dimension of $d$ and comprises $K$-Lipschitz functions.
Let
\begin{equation}
    \epsilon^* = \min_{h \in \gH} \epsilon_s(h) + \epsilon_t(h).
\end{equation}
Let $N_s$ data points be sampled from the same marginal distribution $\mu_s$ with a dependency graph denoted as $\gG_\textup{dep}$, for which $\Lambda(\gG_\textup{dep})$ denotes its forest complexity. Then with a probability of at least $1-\delta$,
    \begin{equation}\label{eq:bd_graph}
    \begin{aligned}
        \epsilon_t(h) &\leq \hat{\epsilon}_s(h) + \sqrt{\frac{8d \log(e N_s / d)}{N_s}} + \\ 
        &\quad \sqrt{\frac{2 \log(2/\delta) \Lambda(\gG_\textup{dep})}{N_s^2}} + 2K W_1 (\mu_s, \mu_t) + \epsilon^*.
    \end{aligned}
    \end{equation}
\end{theorem}

The bound established in Theorem \ref{thm:dep_graph} calls for bounding $\Lambda(\gG_\textup{dep})$ for achieving effective generalization. Importantly, when the node features are independent, the dependency graph has no edge, which implies $\Lambda(\gG_\textup{dep}) = N_s$, the number of nodes. Even if features are not independent, it is important to decorrelate them as much as possible. For instance, if the dependency graph is very sparse such as a tree or a cycle, then $\Lambda(\gG_\textup{dep})$ is still $O(N_s)$. However, if the dependency graph is a grid graph, which still has some sparsity, $\Lambda(\gG_\textup{dep})$ already becomes $O(N_s^{3/2})$, which significantly hinders the generalization bound by increasing it from $O(N_s^{-1/2})$ to $O(N_s^{-1/4})$. This result highlights the significance of reducing interdependencies for representations among adjacent nodes.

\begin{table*}[t]
    \centering
    \vskip 1em
    \caption{Silhouette score ($\uparrow$) for D $\rightarrow$ C and A $\rightarrow$ C (best in \textbf{bold}; second-best \underline{underlined})}
    \scalebox{0.8}{
    \label{tab:silh_score}
    \begin{tabular}{cccccccccc}
        \toprule
          & DANN & UDAGCN & AdaGCN & StruRW & GRADE & SpecReg & A2GNN &PairAlign& DFT (ours)\\
        \midrule
        D $\rightarrow$ C & $0.0397\pm 0.0028$ & $0.0563\pm 0.0024$ & $0.0919\pm 0.0107$ &$-0.0170\pm 0.0016$ & $0.0290\pm 0.0029$ & $0.0563\pm 0.0017$ &$\underline{0.1150}\pm 0.0010$ &$-0.0183\pm 0.0021$& $\boldsymbol{0.2224}\pm 0.0071$ \\
        A $\rightarrow$ C & $0.0764\pm 0.0103$ & $0.1072\pm 0.0017$ & $0.0802\pm 0.0213$ &$-0.0184\pm 0.0026$ & $0.0395\pm 0.0078$ & $0.0665\pm 0.0012$ & $\underline{0.1494}\pm 0.0031$&$-0.0250\pm 0.0030$&$\boldsymbol{0.2608}\pm 0.0017$ \\
        \bottomrule
    \end{tabular}}
\end{table*}

\section{Proofs}\label{app:proofs}
\subsection{Proof of Theorem \ref{thm:dep_graph_markov} and Theorem \ref{thm:dep_graph}}\label{app:proof_inequality}
To prove Theorem \ref{thm:dep_graph_markov} and Theorem \ref{thm:dep_graph}, we first review relevant McDiarmid-type inequalities given in \cite[Corollary~2.10]{paulin2015concentration} and \cite[Theorem~3.6]{zhang2019mcdiarmid} respectively as follows.
\begin{lemma}\label{thm:mcdiarmid}
Given $\vc = (c_1, \cdots, c_n)^\T \in \R^n$. Suppose that $f$ is a real-valued function on $\R^n$ that satisfies the $\vc$-Lipschitz condition: for any $\rvx = (\rx_1, \cdots, \rx_n)^\T$ and $\rvx' = (\rx'_1, \cdots, \rx'_n)^\T$,
\begin{equation}\label{eq:def_cLip}
    \abs{f(\rvx) - f(\rvx')} \leq \sum_{i=1}^n c_i \vone_{\{\rx_i \neq \rx'_i\}}.
\end{equation}
\begin{enumerate}
    \item \cite{paulin2015concentration} Let $t_{\textup{mix}}$ be the mixing time for the markov chain $\rvx = (\vx_1, \cdots, \vx_n)$. Then given $\epsilon > 0$, it holds that
    \begin{equation}\label{eq:mixing}
        P \left( f(\rvx) - \E [f(\rvx)] \geq \epsilon \right) \leq \exp \left( - \frac{2 \epsilon^2}{9 t_{\textup{mix}} \norm{\vc}_\infty^2} \right).
    \end{equation}
    \item \cite{zhang2019mcdiarmid} Let $\gG_\textup{dep}$ be the dependency graph of a random variable $\rvx$. Then given $\epsilon > 0$, it holds that
    \begin{equation}\label{eq:mcdiarmid}
        P \left( f(\rvx) - \E [f(\rvx)] \geq \epsilon \right) \leq \exp \left( - \frac{2 \epsilon^2}{\Lambda(\gG_\textup{dep}) \norm{\vc}_\infty^2} \right).
    \end{equation}
\end{enumerate}
\end{lemma}
Now we are ready to present our proof.
\begin{proof}
We only need to prove \eqref{eq:bd_graph}. Once we establish \eqref{eq:bd_graph}, replacing $\Lambda(\gG_\textup{dep})$ with $9 t_{\textup{mix}}$ will yield \eqref{eq:bd_mixing}. First, we recall a generalization bound when using the Wasserstein-1 distance as follows. Consider a hypothesis class $\gH$ comprising $K$-Lipschitz functions. For any $h \in \gH$, standard domain adaptation analysis (e.g., in slightly various forms in \cite{redko2017theoretical, shen2018wasserstein, you2023graph}) yields 
\begin{equation}\label{eq:previous_bd}
    \epsilon_t(h) \leq  \epsilon_s(h) + 2KW_1 (\mu_s, \mu_t) + \epsilon^*.
\end{equation}

We proceed to bound $\epsilon_s(h)$. According to statistical learning theory (e.g., \cite[Ch.~28.1]{shalev2014understanding}), given any sample $\gS$ of data points in the source domain, the Rademacher complexity of $\epsilon_s \circ \gH \circ \gS$, denoted as $R(\epsilon_s \circ \gH \circ \gS)$, satisfies
\begin{equation}
    R(\epsilon_s \circ \gH \circ \gS) \leq \sqrt{\frac{2d \log(e N_s d)}{N_s}}.
\end{equation}
Define 
\begin{equation}
    \Delta := \sup_{h \in \gH} \epsilon_s(h) - \hat{\epsilon}_s(h).
\end{equation}
Clearly, by definition of the Rademacher complexity, 
\begin{equation}\label{eq:bd_EDelta}
    \E [\Delta] \leq 2 \E_\gS \left[ R(\epsilon_s \circ \gH \circ \gS) \right] \leq \sqrt{\frac{8d \log(e N_s d)}{N_s}}.
\end{equation}
It remains to bound the deviation of $\Delta - \E [\Delta]$. We proceed by applying \eqref{eq:mcdiarmid} with $f = \Delta$. In this case, we need to determine the appropriate Lipschitz constant $\vc$ for which $\Delta$ is $\vc$-Lipschitz. To do this, consider that $\epsilon_s(h) - \hat{\epsilon}_s(h)$ may differ by at most $2/N_s$. Therefore, \eqref{eq:def_cLip} holds with $c_1 = \cdots = c_n = 2/N_s$. Consequently, we conclude that $\Delta$ is a $(2/N_s)\vone$-Lipschitz function.

Given $\vc = (2/N_s)\vone$, we have $\norm{\vc}_\infty = 2/N_s$. Setting
\begin{equation}
    \epsilon = \frac{2}{N_s} \sqrt{\frac{\log(2/\delta) \Lambda(\gG_\textup{dep})}{2}}
\end{equation}
in \eqref{eq:mcdiarmid} yields that with probability at least $1-\delta$,
\begin{equation}\label{eq:bd_DeltaEDelta}
    \Delta - \E \Delta \leq \sqrt{\frac{2 \log(2/\delta) \Lambda(\gG_\textup{dep})}{N_s^2}}.
\end{equation}

Finally, combining \eqref{eq:previous_bd}, \eqref{eq:bd_EDelta} and \eqref{eq:bd_DeltaEDelta} yields \eqref{eq:bd_graph}.
\end{proof}

\begin{figure*}[t]
    \centering
    \includegraphics[width=.85\textwidth]{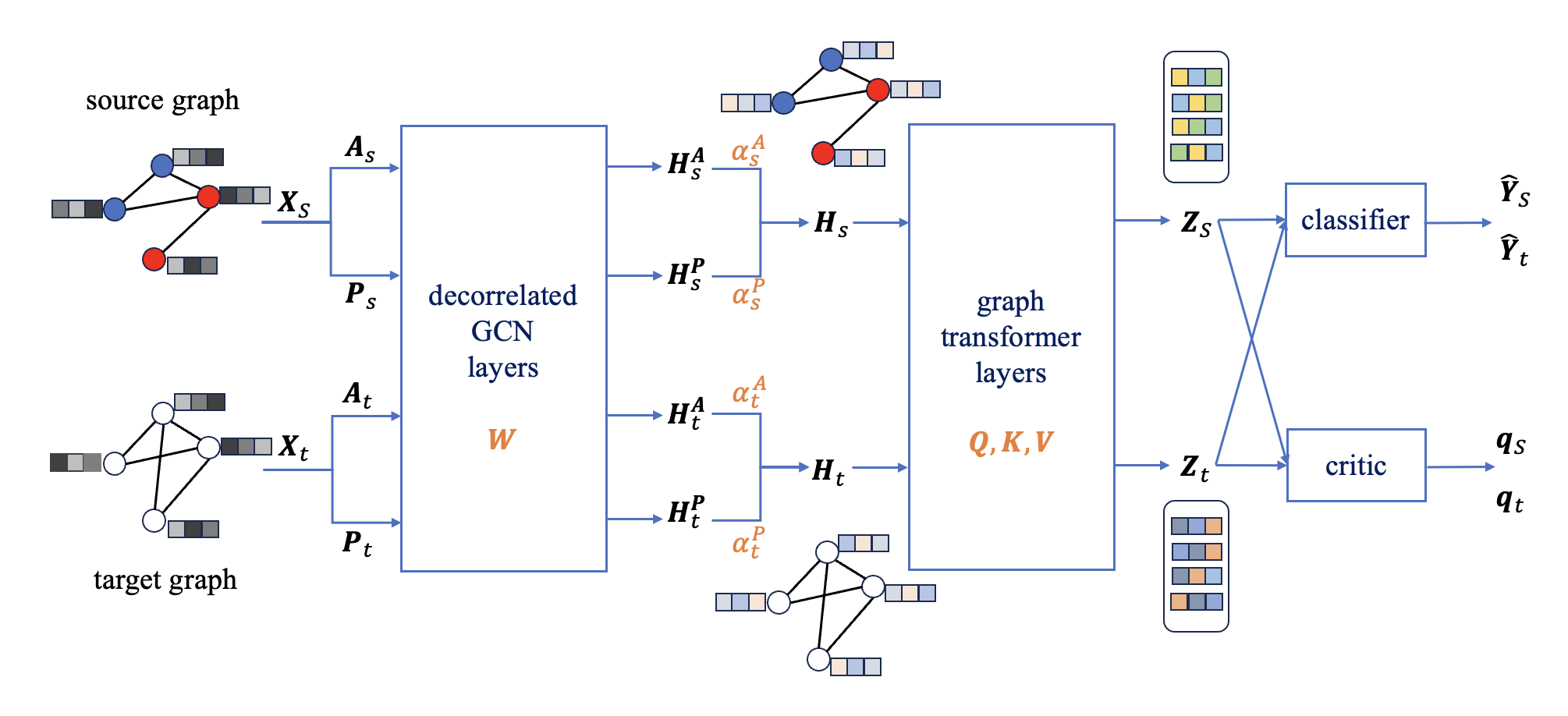}
    \caption{The network architecture of DFT.}
    \label{fig:model}
\end{figure*}

\subsection{Proof of Theorem \ref{thm:high_corr}}\label{app:proof_energy}
\begin{proof}
For simplicity we denote $\mB = \tilde{\mA}^{2k}$. We express the left-hand side of \eqref{eq:exp_HHT} as
\begin{align*}
    &\quad \E \left[ \norm{\mH^{(k)} \mH^{(k) \T}}_{\operatorname{F}}^2 \right] \\
    &= \E \left[ \tr ( \mX^\T \mB \mX )^2 \right] \\
    &= \E \left[ \sum_{i=1}^N \sum_{j=1}^N \sum_{p=1}^N \sum_{q=1}^N \sum_{l=1}^D \sum_{m=1}^D  X_{il} B_{ij} X_{jm} X_{pm} B_{pq} X_{ql} \right].
\end{align*}
Since the entries of $\mX$ are i.i.d. sampled from $\gN(0,1)$, the expectation of $X_{il} B_{ij} X_{jm} X_{pm} B_{pq} X_{ql}$ is $0$ unless one of the following cases happens: (1) $i=q$, $j=p$, but $l \neq m$; (2a) $l=m$, $i=q$, $j=p$, but $i \neq j$; (2b) $l=m$, $i=j$, $p=q$, but $i \neq p$; (2a) $l=m$, $i=p$, $j=q$, but $i \neq j$; (3) $l=m$ and $i=j=p=q$. 

Note that $\E[\rx^2] = 1$ and $\E[\rx^4] = 3$ for $\rx \sim \gN(0,1)$. Therefore, we have the following:
\begin{itemize}
    \item case (1), $\E[X_{il} B_{ij} X_{jm} X_{pm} B_{pq} X_{ql}] = B_{ij}B_{ji}$
    \item case (2a), $\E[X_{il} B_{ij} X_{jm} X_{pm} B_{pq} X_{ql}] = B_{ij}B_{ji}$
    \item case (2b), $\E[X_{il} B_{ij} X_{jm} X_{pm} B_{pq} X_{ql}] = B_{ii}B_{pp}$
    \item case (2c), $\E[X_{il} B_{ij} X_{jm} X_{pm} B_{pq} X_{ql}] = B_{ij}B_{ij}$
    \item case (3), $\E[X_{il} B_{ij} X_{jm} X_{pm} B_{pq} X_{ql}] = 3B_{ii}$
\end{itemize}

Therefore, the above expectation reduces to
\begin{align*}
    &\quad \E \left[ \norm{\mH^{(k)} \mH^{(k) \T}}_{\operatorname{F}}^2 \right] \\
    &= D(D-1) \sum_{i,j} B_{ij} B_{ji} + D \sum_{i \neq j} B_{ij} B_{ji} + \\
    &\qquad D \sum_{i \neq p} B_{ii} B_{pp} + D \sum_{i \neq j} B_{ij} B_{ij} + D \sum_{i} 3B_{ii} \\
    &= D(D-1) \sum_{i,j} B_{ij}^2 + D \sum_{i \neq j} B_{ij}^2 + \\
    &\qquad D \sum_{i \neq p} B_{ii} B_{pp} + D \sum_{i \neq j} B_{ij}^2 + D \sum_{i} 3B_{ii}^2 \\
    &= D^2 \sum_{i,j} B_{ij}^2 + \\
    &\qquad D \sum_{i \neq j} B_{ii} B_{jj} + D \sum_{i \neq j} B_{ij}^2 + D \sum_{i} 2B_{ii}^2 \\
    &= D^2 \sum_{i,j} B_{ij}^2 + D \sum_{i, j} B_{ii} B_{jj} + D \sum_{i, j} B_{ij}^2 \\
    &= D \sum_{i,j} (D+1) B_{ij}^2 +  B_{ii} B_{jj}.
\end{align*}
Now that we have established \eqref{eq:exp_HHT}. Since all the terms are positive, we only need to show that $(\tilde{\mA}^{2k})_{ij}$ increases with $k$ for any $i,j$. Note that $(\tilde{\mA}^{2k})_{ij}$ is the number of different paths of length $2k$ that traverses from node $i$ to node $j$. Since $\tilde{\mA} = \mA + \mI$ is the adjacency matrix with self-loops, a path of length $2k$ can be extended to a path of length $2(k+1)$ by adding two self-loops. Consequently, the number of paths of length $2(k+1)$ is larger than that of paths of length $2k$. Therefore, each entry of $\tilde{\mA}^{2k}$ increases with $k$.

\end{proof}

\section{Silhouette score}\label{app:silhouette}
To gain deeper insight into the quality of the learned target-domain representations, we adopt the Silhouette Score~\cite{rousseeuw1987silhouettes} as an unsupervised metric to measure the compactness and separability of different classes in the target feature space. A higher Silhouette Score indicates that data points are well clustered within the same class and are well separated from other classes. We compare our method against several baselines using this metric. The results, summarized in Table~\ref{tab:silh_score}, demonstrate that our approach consistently achieves better clustering quality across datasets.

\newpage

\section{More Details of the Model}\label{app:methodology}
We present more details of DFT for reproducibility. The overall architecture is illustrated in Figure~\ref{fig:model}. 

\subsection{Network Architecture}\label{app:network}
General feature-based domain adaptation methods have widely followed the DANN framework introduced by \citet{ganin2016domain}. In the DANN framework, a shared feature extractor is applied to both the source and target domain. The objective is to learn domain-invariant representations in both domains. Subsequently, the source domain representation is sent through a source classifier, while both source and target domain representations are sent through a domain classifier. Based on DANN, the UDAGCN framework, introduced by \citet{wu2020unsupervised}, has had a substantial impact on the domain adaptation for graph data and has served as a foundation for subsequent research such as~\cite{you2023graph}. 

In addition to DANN, UDAGCN incorporates both adjacency-based and random walk-based information while leveraging an inter-graph attention scheme. Specifically, for $r=s,t$, UDAGCN first creates point-wise mutual information (PPMI) matrices $\mP_r$ from random walk samplers on the graphs. Subsequently, GCN layers are employed to process the input features $\mX_r$ through message passing, leveraging both $\mA_r$ and $\mP_r$, resulting in the generation of latent representations $\mH_r^\mA$ and $\mH_r^\mP$. Following this, attention coefficients $\alpha_r^\mA$ and $\alpha_r^\mP$ are learned as output of a separate neural network. Finally, the aggregated representation $\mH_r$ for each domain is obtained via the following expression:
\begin{equation}
    \mH_r = \alpha_r^\mA \mH_r^\mA + \alpha_r^\mP \mH_r^\mP, \quad r = s, t. \label{attention}
\end{equation}
In our case, for each domain $r=s,t$, we have two features $(\mH_r^\mA)^{(L)}$ and $(\mH_r^\mP)^{(L)}$ following $L$ layers of decorrelated feature extraction as in~\eqref{eq:grad_des_simple}. These are further processed by the graph transformer layers to produce features $\mZ_r$, $r=s,t$.

The representations obtained by the transformer layer are then fed into a classifier, which is implemented as a dense neural network denoted by $f_\textup{clf}$ with softmax activations. The dense network is shared across both source and target domains. This classifer produces output vectors $\hat{\mY}_r = f_\textup{clf}(\mZ_r) \in \R^{N_r \times C}$, $r = s,t$.

The upper bounds in \eqref{eq:bd_graph} and \eqref{eq:bd_mixing} both suggest minimizing the Wasserstein-1 distance between representation distributions in source and target domains. To this end, similarly to \cite{dai2022graph}, we adopt an adversarial training approach. Specifically, we take a dense neural network $f_\textup{critic}$ as the critic for distinguishing representations from source and target domains, which produces a real number for each node: $\vq_r = f_\textup{critic}(\mZ_r) \in \R^{N_r}$, $r = s,t$.

\subsection{Loss Functions}\label{app:loss}
We describe the loss terms in detail as follows.
The source classifier loss $\gL_s$ is a standard cross-entropy loss for the labeled data in the source domain:
\begin{equation}
    \gL_s = -\frac{1}{N_s}\sum_{i=1}^{N_s} \vy_{s,i}^\T \log(\hat{\vy}_{s,i}),
\end{equation}
where $\vy_{s,i}$ is the one-hot class label corresponding to the $i$-th node feature for the source graph, and $\hat{\vy}_{s,i}$ is the $i$-th row of $\hat{\mY}_s$. Similarly, the target classifier loss $\gL_t$ is given by
\begin{equation}
    \gL_t = -\frac{1}{N_t}\sum_{i=1}^{N_t} \hat{\vy}_{t,i}^\T \log(\hat{\vy}_{t,i}),
\end{equation}
where only $\hat{\vy}_{t,i}$, the $i$-th row of $\hat{\mY}_t$ is used to calculate the entropy term due to the unavailability of target labels.

The critic loss $\gL_\textup{critic}$ and the penalty term $\gL_{\textup{gp}}$ inherit the loss function of the Wasserstein GAN in \cite{gulrajani2017improved}. Specifically,
\begin{equation}
    \gL_\textup{critic} = \frac{1}{N_s}\sum_{i=1}^{N_s} q_{s,i} - \frac{1}{N_t}\sum_{i=1}^{N_t} q_{t,i},
\end{equation}
where $q_{s,i}$ is the $i$-th entry of $\vq_s$ and $q_{t,i}$ is the $i$-th entry of $\vq_t$. Moreover, the gradient penalty ensures 1-Lipschitz condition with
\begin{equation}
\begin{aligned}
    \gL_\textup{gp} &= \sum_{r=s,t} \frac{1}{N_r} \sum_{i=1}^{N_r} \left( \norm{ \nabla f_\textup{critic} \left( z_{r,i} \right) }_2 - 1 \right)^2.
\end{aligned}
\end{equation}

\section{Validation of Covariate Shift}\label{app:covariate_shift}
To validate the covariate shift assumption, we consider the following visualization results. 

First, we examine $p(y|\rvx)$, where $\rvx$ represents raw node features, in both source and target domains for the ``Citation $\rightarrow$ DBLP'' scenario. Although we do not have access to the true values of $p(y|\rvx)$, we estimate them by a nearest-neighbor approach. For each feature among the samples, we find its 128 nearest neighbors in the source domain and assign $y_s$ based on a majority vote of their labels. The same process is applied in the target domain to obtain $y_t$. We estimate $p(y_s=y_t)$ by calculating the fraction of samples $\rvx$ with $y_s=y_t$. Across all samples in both domains, we find $p(y_s=y_t) \approx 0.70$. In comparison, when we randomly apply a label shuffling in the target domain, we obtain $p(y_s=y_t) \approx 0.22$. The substantially higher probability in the original setup highlights the presence of covariate shift even in the raw node feature space.

Second, for scenarios ``Citation $\rightarrow$ DBLP'' and ``Blog1 $\rightarrow$ Blog2'', we uses t-SNE to visualize the raw node features of two dominant classes in Figure~\ref{fig:conv_vis}. We randomly sample 10\% of the nodes from the Citation dataset and 20\% from the Blog dataset for better visibility. Since we use raw features, the two classes are not well-separated in the figures. Nevertheless, nearby points still show similar colors no matter they are from she source (``x'') or target (``o'') domain. We do not observe any noticeable shift of $p(y|\rvx)$ between the two domains from the figures.

\begin{figure*}[t]
  \centering
  \subfigure[Citation $\rightarrow$ DBLP]{
    \label{fig:cov_citation_dblp}
    \includegraphics[width=0.4\textwidth]{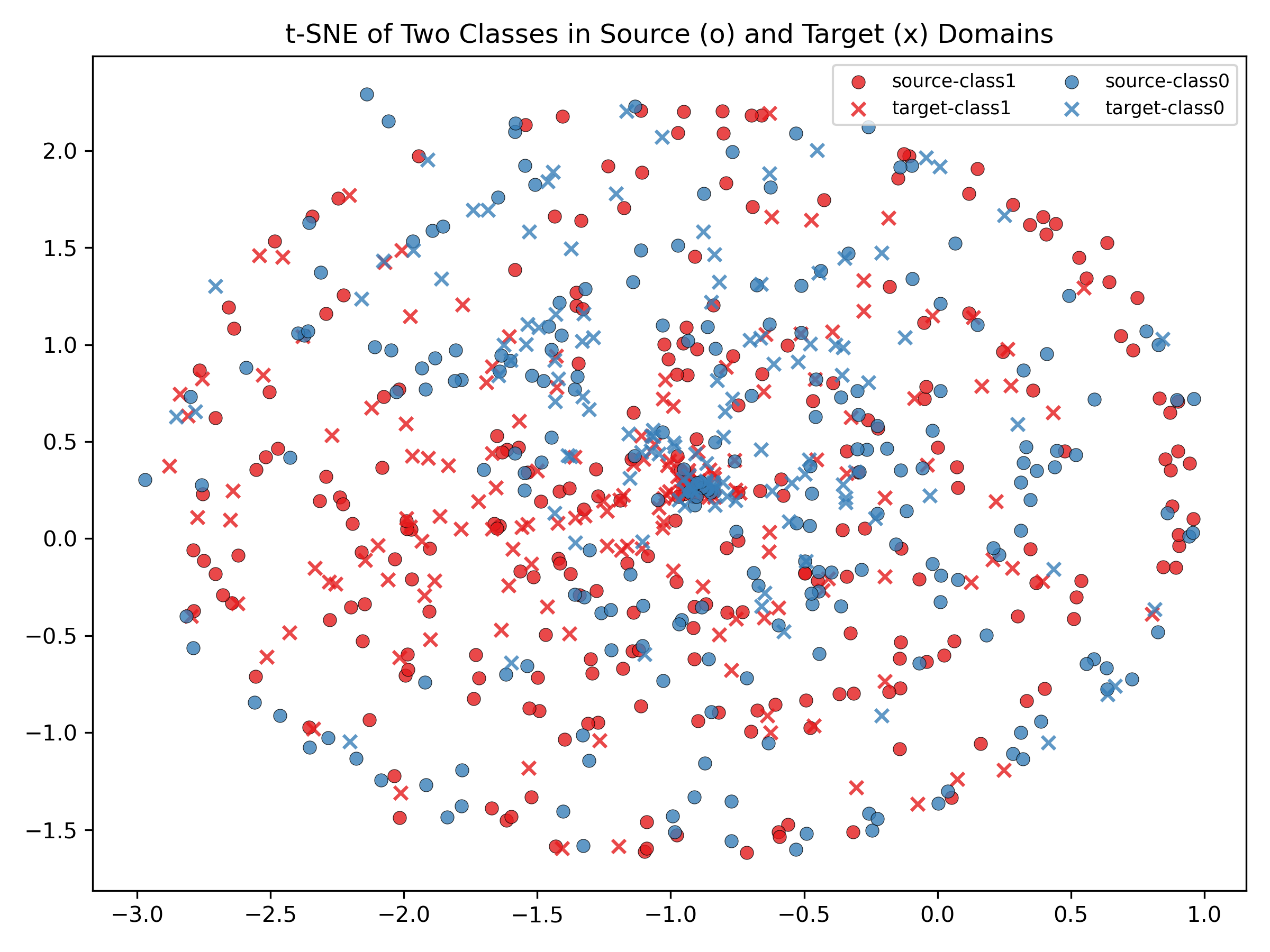}
  }
  \subfigure[Blog1 $\rightarrow$ Blog2]{
    \label{fig:cov_blog}
    \includegraphics[width=0.4\textwidth]{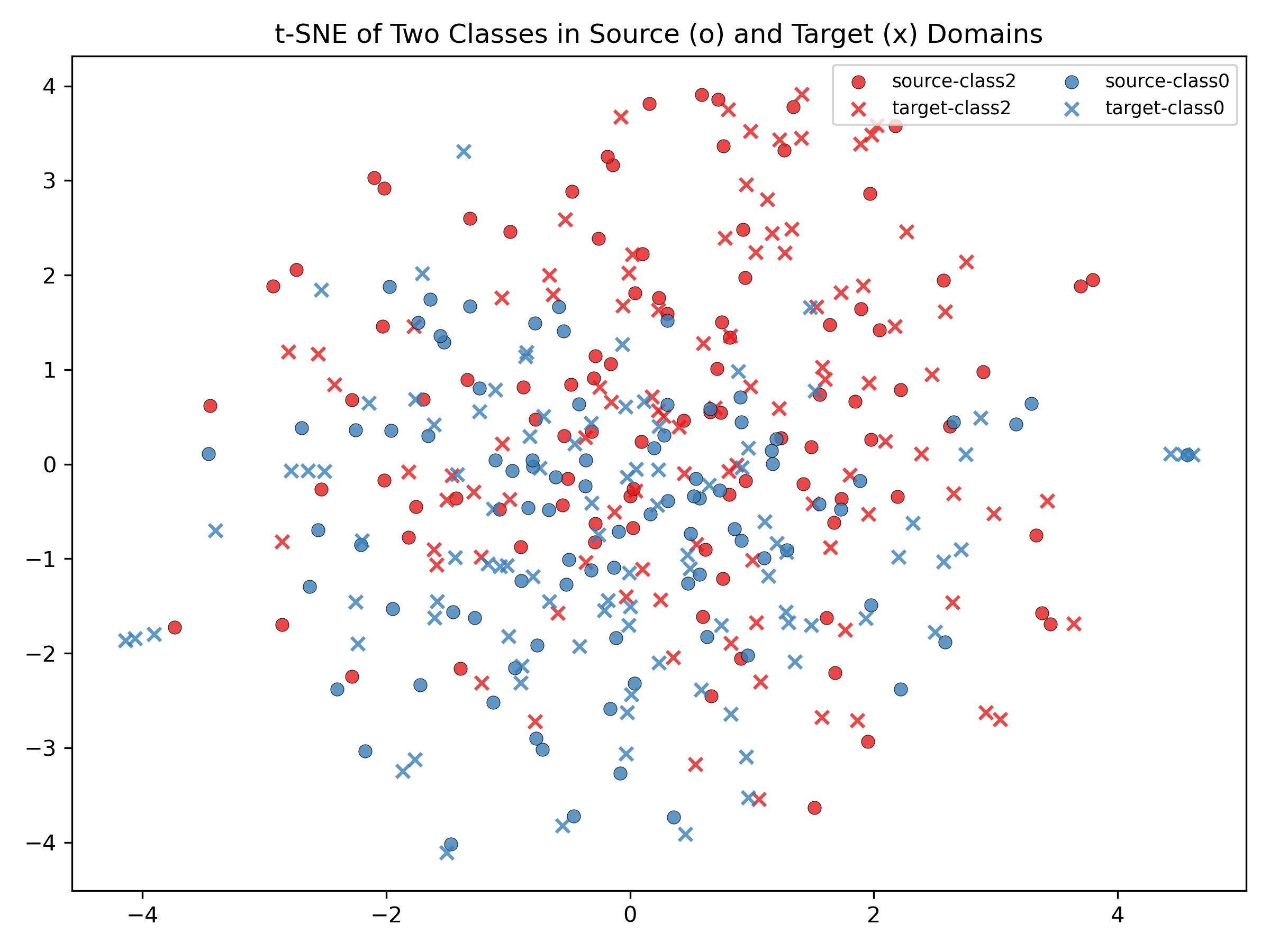}
  }
  \caption{Visualization of raw node features. We choose two dominant classes in two scenarios, labeled with different colors. Nodes in the source and target domains are labeled with ``x'' and ``o'', respectively.}
  \label{fig:conv_vis}
\end{figure*}

Finally, we also distinguish learned representations from the source (``x'') and target (``o'') domains using t-SNE. Specifically, we perform t-SNE on the features learned by the model for two domain adaptation tasks: DBLP (D) $\rightarrow$ Citation (C) and ACM (A) $\rightarrow$ Citation (C). To make the visualization clearer, we only randomly keep 20\% of the nodes from both the source and target domains before applying t-SNE. The results are shown in Figure~\ref{fig:tsne_source_target}. In the learned distributions, we cannot distinguish the two domains at all, whether considering the marginal or conditional distributions. This visualization confirms that our method has successfully learned a shared representation space, which has aligned the distributions on the two domains.

\begin{figure*}[t]
  \centering
  \subfigure[A $\rightarrow$ C]{
    \label{fig:tsne_acm_citation}
    \includegraphics[width=0.32\textwidth]{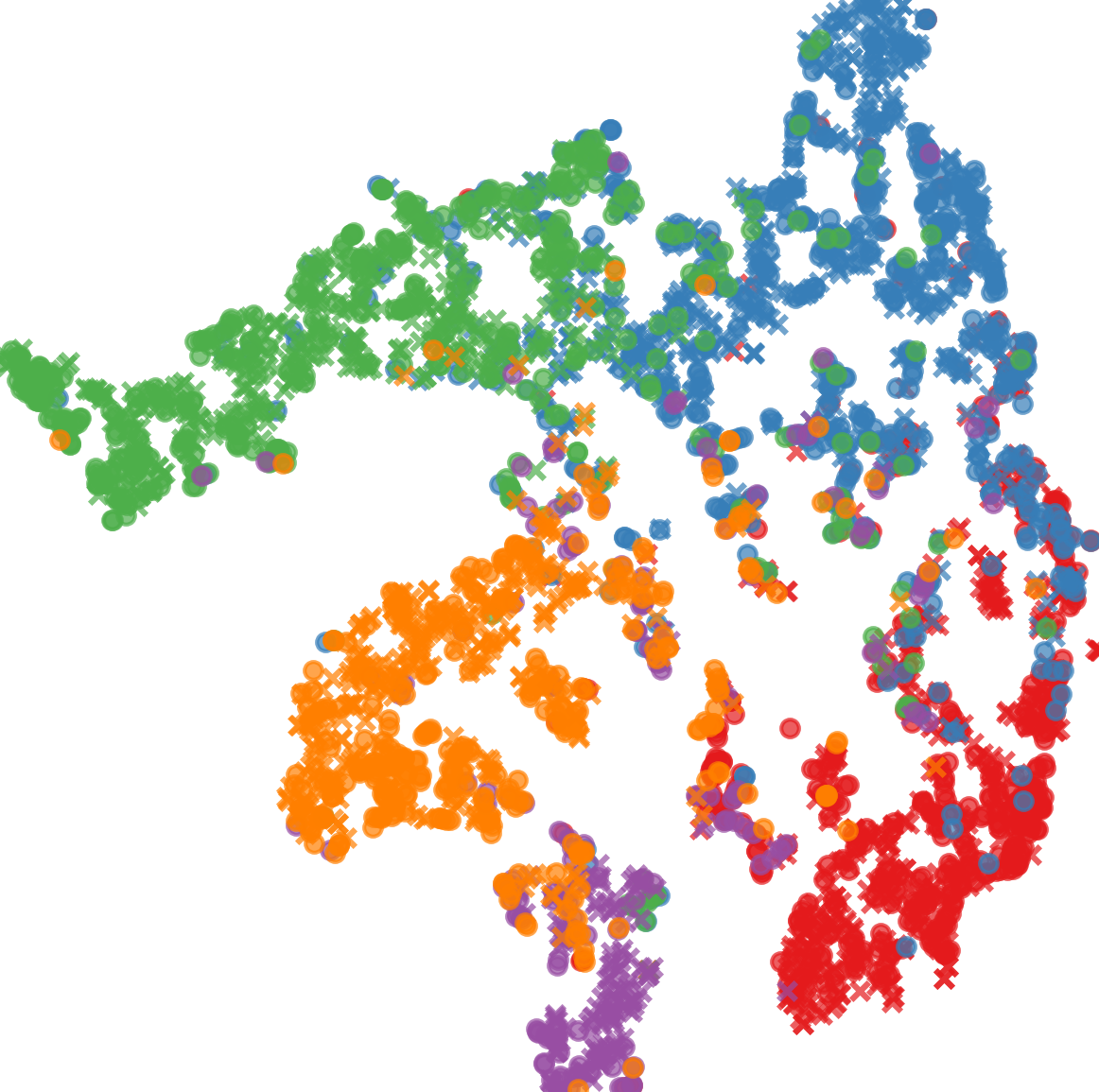}
  }
  \subfigure[D $\rightarrow$ C]{
    \label{fig:tsne_dblp_citation}
    \includegraphics[width=0.32\textwidth]{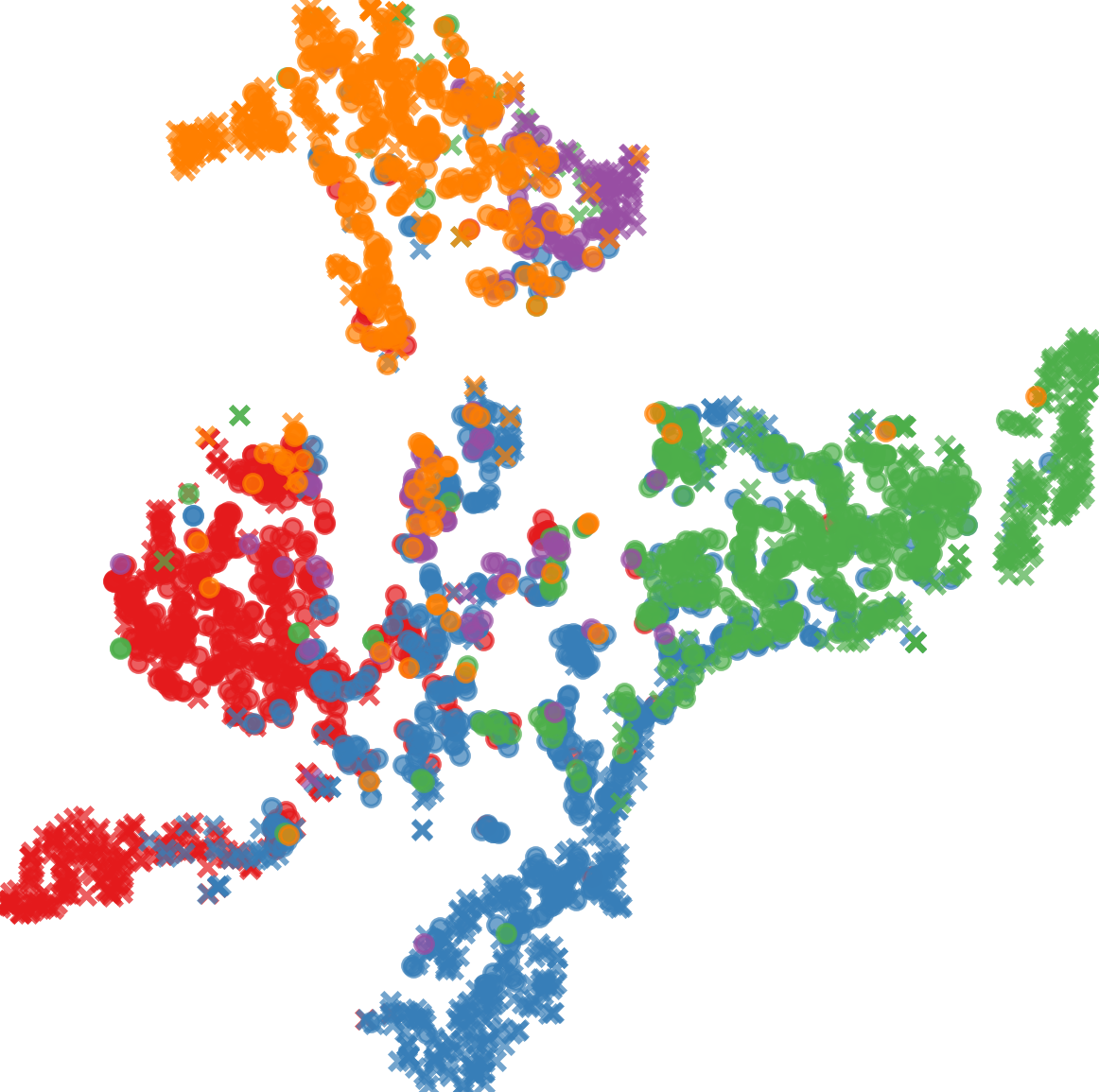}
  }
  \caption{Visualization of learned representations. Nodes in the source and target domains are labeled with ``x'' and ``o'', respectively.}
  \label{fig:tsne_source_target}
\end{figure*}

\section{More Experimental Analysis}\label{app:more_exp}
\subsection{Large-Scale Dataset}\label{app:largedataset}
To demonstrate the generalizability of our method on large-scale datasets, we conduct experiments on the MAG dataset from the Open Graph Benchmark (OGB)~\cite{hu2020open}, following the experimental setup in~\cite{liu2024pairwise}. In the experiments, papers from different countries are treated as distinct graph domains. Dataset statistics are summarized in Table~\ref{tab:mag}. Notably, the label distributions across domains exhibit substantial discrepancies, as shown in Table~\ref{tab:label_dist}, indicating a significant domain shift between the source and target graphs.

The GDA results are shown in Table~\ref{tab:mag_results}. Our method clearly excels baseline methods.

\begin{table}[t]
    \centering
    \caption{Statistics of MAG datasets}
    \label{tab:mag}
    \scalebox{0.9}{
    \begin{tabular}{c|cccc}
        \toprule
        Dataset & \#Nodes & \#Edges & \#Attributes & \multicolumn{1}{c}{\#Labels} \\
        \midrule
        MAG\_FR & 29,262 & 78,222 & 128 & 20 \\
        MAG\_RU & 32,833 & 67,994 & 128 & 20 \\
        MAG\_JP & 37,498 & 90,944 & 128 & 20 \\
        \bottomrule
    \end{tabular}
    }
\end{table}

\begin{table*}[t]
\vskip 1em
\caption{Label distribution (\%) in source (FR) and target domains (RU)}
\label{tab:label_dist}
\centering
\scalebox{0.85}{
\begin{tabular}{|c|*{20}{c|}c|}
\hline
\textbf{Label} & 0 & 1 & 2 & 3 & 4 & 5 & 6 & 7 & 8 & 9 & 10 & 11 & 12 & 13 & 14 & 15 & 16 & 17 & 18 & 19 & Total \\
\hline
\textbf{Source (\%)} & 2.59 & 4.42 & 4.18 & 2.98 & 3.35 & 2.44 & 11.09 & 1.64 & 2.43 & 2.49 & 1.73 & 1.08 & 2.19 & 0.57 & 0.87 & 1.16 & 1.05 & 0.89 & 0.43 & 52.44 & 100.00 \\
\hline
\textbf{Target (\%)} & 0.38 & 1.24 & 1.77 & 0.87 & 1.29 & 0.94 & 0.39 & 1.05 & 1.21 & 0.80 & 0.50 & 1.15 & 0.31 & 0.72 & 0.45 & 0.60 & 0.53 & 0.38 & 0.27 & 85.16 & 100.00 \\
\hline
\end{tabular}}
\end{table*}

\begin{table*}[t]
    \centering
    \vskip 1em
    \caption{Performance (\%) on MAG domain adaptation tasks}
    \label{tab:mag_results}
    \scalebox{0.9}{
    \begin{tabular}{c|ccccccccc}
        \toprule
        Task & DANN & AdaGCN & UDAGCN & StruRW & SpecReg & GRADE & A2GNN & PairAlign &DTW \\
        \midrule
        $FR \rightarrow RU$ & 35.07 & 77.34 & 78.45 & 77.48 & 75.33 & 81.75 & 68.51 & 20.76 & \textbf{85.15} \\
        $FR \rightarrow JP$ & 33.98 & 55.71 & 55.38 & 56.54 & 56.20 & 57.71 & 58.55 & 26.00 & \textbf{60.05} \\
        \bottomrule
    \end{tabular}
    }
\end{table*}

\subsection{Ablation of high-homophily and low-homophily graphs}\label{app:homo}
To test sensitivity to homophily, we consider the Task $A \rightarrow D$ in Citation dataset. We randomly remove edges in both source and target domains and have four variations: 
\begin{itemize}
    \item Remove 10\% edges connecting nodes of different labels.
    \item Remove 10\% edges connecting nodes of same labels.
    \item Remove 30\% edges connecting nodes of different labels.
    \item Remove 30\% edges connecting nodes of same labels.
\end{itemize}
We compare our method with A2GNN (the second-best performing method), as shown in Table~\ref{tab:homophily_effect}. The results indicate that variations in homophily have a relatively limited impact on performance.
\begin{table*}[t]
\caption{Effect of modifying homophily levels on DFT and A2GNN performance}
\label{tab:homophily_effect}
\centering
\begin{tabular}{c|l|c}
\textbf{Setting} & \textbf{Description} & \textbf{DFT/A2GNN (Micro-F1)} \\
\hline
Baseline & Original graph without modification & $74.84 / 66.71$ \\
\hline
Small ↑ Homophily & Remove 30\% of heterophilous edges & $73.84 / 65.64$ \\
\hline
Large ↑ Homophily & Remove 50\% of heterophilous edges & $72.93 / 65.71$ \\
\hline
Small ↓ Homophily & Remove 30\% of homophilous edges & $71.69 / 64.66$ \\
\hline
Large ↓ Homophily & Remove 50\% of homophilous edges & $71.33 / 61.25$ \\
\end{tabular}
\end{table*}

\subsection{Runtime}\label{app:runtime}
We record the average runtime per epoch for the scenario D $\rightarrow$ C using each baseline method, as shown in Table~\ref{tab:time}.

\begin{table*}[t]
    \centering
    \vskip 1em
    \caption{Average time each epoch (s/epoch)}
    \label{tab:time}
    
    \scalebox{1.0}{
    \begin{tabular}{ccccccc}
        \toprule
          & D $\rightarrow$ C & C $\rightarrow$ D & A $\rightarrow$ C & C $\rightarrow$ A & D $\rightarrow$ A & A $\rightarrow$ D \\
        \midrule
        DANN&   $0.0851\pm 0.0017$ &$0.0638\pm 0.0007$ &$0.0900\pm 0.0009$ &$0.0975\pm 0.0009$& $0.0927\pm 0.0057$ & $0.0642\pm 0.0010$ \\
        UDAGCN& $0.2458\pm 0.0025$& $0.2177\pm 0.0031$ & $0.2937\pm 0.0044$ &$0.3002\pm 0.0167$ & $0.2443\pm 0.0072$ & $0.2174\pm 0.0010$  \\
        AdaGCN& $0.2526\pm 0.0023$& $0.2263\pm 0.0043$& $0.2686\pm 0.0062$& $0.2727\pm 0.0035$& $0.2549\pm 0.0034$& $0.2311\pm 0.0018$ \\
        StruRW& $0.1094\pm 0.0008$& $0.1052\pm 0.0013$& $0.1257\pm 0.0009$ &$0.1255\pm 0.0011$& $0.1109\pm 0.0012$ & $0.1090\pm 0.0006$  \\
        GRADE & $0.0886\pm 0.0013$& $0.0641\pm 0.0002$& $0.0938\pm 0.0017$& $0.0998\pm 0.0016$& $0.0926\pm 0.0013$& $0.0645\pm 0.0004$  \\
        SpecReg&$0.1378\pm 0.0063$& $0.1213\pm 0.0004$& $0.1488\pm 0.0016$& $0.1509\pm 0.0002$& $0.1494\pm 0.0015$& $0.1216\pm 0.0009$  \\
        A2GNN& $0.0613\pm 0.0008$&$0.0596\pm 0.0012$ &$0.0670\pm 0.0017$&$0.0678\pm 0.0005$&$0.0623\pm 0.0006$&$0.0584\pm 0.0009$\\ 
        PairAlign&$0.2705\pm 0.0019$ &$0.2750\pm 0.0023$ &$0.3370\pm 0.0049$ &$0.3363\pm 0.0034$ &$0.2766\pm 0.0025$ &$0.2834\pm 0.0053$\\
        DFT& $0.6156\pm 0.0028$& $0.6168\pm 0.0040$& $0.8424\pm 0.0270$& $0.8188\pm 0.0076$&$0.6287\pm 0.0025$& $0.6255\pm 0.0095$ \\
        \bottomrule
    \end{tabular}}
\end{table*}

\subsection{Feature Visualization for A $\rightarrow$ C}\label{app:atoc}
We present feature visualization for the scenario A $\rightarrow$ C in Figure~\ref{fig:acm_citation}. The corresponding ICDR has been presented in the main text.

\begin{figure*}[t]
  \centering
  \subfigure[DANN]{
    \label{fig:subfig:b2}
    \includegraphics[width=0.3\textwidth]{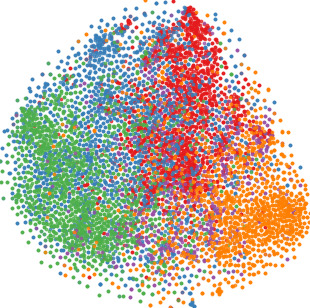}
  }
  \hfill
  \subfigure[UDAGCN]{
    \label{fig:subfig:c2}
    \includegraphics[width=0.3\textwidth]{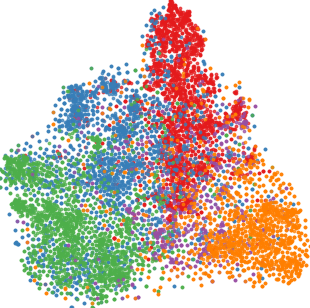}
  }
  \hfill
  \subfigure[AdaGCN]{
    \label{fig:subfig:d2}
    \includegraphics[width=0.3\textwidth]{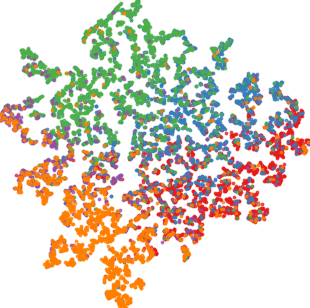}
  }\\
  \subfigure[GRADE]{
    \label{fig:subfig:e2}
    \includegraphics[width=0.3\textwidth]{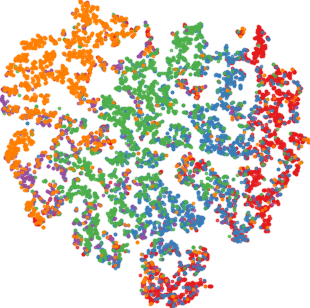}
  }
  \hfill
  \subfigure[StruRW]{
    \label{fig:subfig:f2}
    \includegraphics[width=0.3\textwidth]{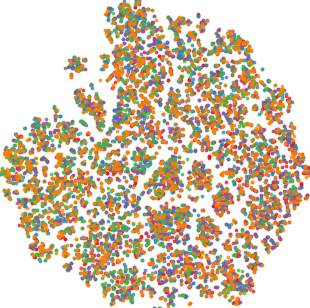}
  }
  \hfill
  \subfigure[SpecReg]{
    \label{fig:subfig:g2}
    \includegraphics[width=0.3\textwidth]{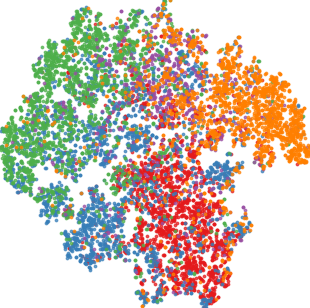}
  }\\
  \subfigure[A2GNN]{
    \label{fig:subfig:h2}
    \includegraphics[width=0.3\textwidth]{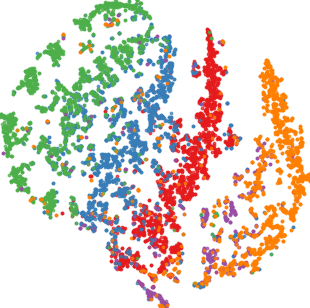}
  }
  \hfill
  \subfigure[PairAlign]{
    \label{fig:subfig:i2}
    \includegraphics[width=0.3\textwidth]{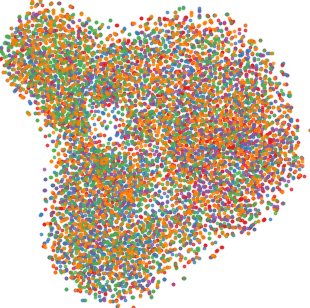}
  }
  \hfill
  \subfigure[DFT]{
    \label{fig:subfig:a2}
    \includegraphics[width=0.3\textwidth]{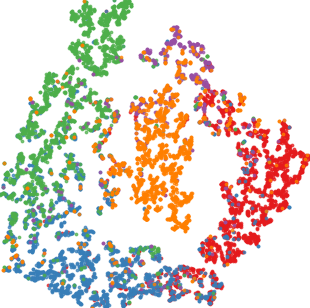}
  }
  \caption{Visualization of the representations learned in all methods (A $\to$ C).}
  \label{fig:acm_citation}
\end{figure*}

\end{document}